\documentclass[11pt]{article}
\usepackage[T1]{fontenc}
\usepackage[utf8]{inputenc}
\usepackage{lmodern}

\usepackage[
    letterpaper,
    margin=1in
]{geometry}

\usepackage{amsmath}
\usepackage{amssymb}
\usepackage{amsfonts}
\usepackage{amsthm}
\usepackage{mathtools}
\usepackage{bm}

\usepackage{graphicx}
\usepackage{subcaption}
\graphicspath{{assets/}}
\usepackage{booktabs}
\usepackage{multirow}
\usepackage{array}
\usepackage{tabularx}
\usepackage{float}
\usepackage{placeins}

\usepackage{microtype}
\usepackage{authblk}

\usepackage{enumitem}

\usepackage{xcolor}
\usepackage{xurl}

\usepackage[numbers,sort&compress]{natbib}

\usepackage[
    colorlinks=true,
    linkcolor=blue,
    citecolor=blue,
    urlcolor=blue
]{hyperref}

\usepackage[nameinlink,capitalise,noabbrev]{cleveref}
\newfloat{algorithm}{tbp}{loa}
\floatname{algorithm}{Algorithm}
\floatstyle{ruled}\restylefloat{algorithm}
\crefname{algorithm}{algorithm}{algorithms}
\Crefname{algorithm}{Algorithm}{Algorithms}

\newtheorem{proposition}{Proposition}

\theoremstyle{definition}

\theoremstyle{remark}

\newcommand{\E}{\mathbb{E}}
\newcommand{\Pp}{\mathbb{P}}
\newcommand{\R}{\mathbb{R}}

\newcommand{\ind}{\mathbf{1}}

\DeclareMathOperator{\Var}{Var}

\newcolumntype{Y}{>{\centering\arraybackslash}X}

\title{
    \textbf{GeoACE: Geometry-Diverse Anchor--Correction Ensembles
    for Heterogeneous Treatment Effect Estimation}
}

\author[1]{Ali Haghpanah Jahromi\thanks{Corresponding author: \href{mailto:alihgpj@gmail.com}{alihgpj@gmail.com}}}
\author[1]{Mohammad Taheri\thanks{\href{mailto:motaheri@shirazu.ac.ir}{motaheri@shirazu.ac.ir}}}
\author[1]{Zohreh Azimifar\thanks{\href{mailto:azimifar@cse.shirazu.ac.ir}{azimifar@cse.shirazu.ac.ir}}}
\affil[1]{Department of Electrical and Computer Engineering, University of Shiraz, Shiraz, Iran}

\date{}

\begin{document}

\maketitle


\begin{abstract}

Predicting heterogeneous causal effects supports informed decision-making in medicine, public policy, 
and economics, yet suitable inductive biases vary with treatment overlap, covariate imbalance, prognostic 
structure, and sample size. The Geometry-Diverse Anchor--Correction Expert Ensemble (GeoACE) addresses this 
variation through five complementary experts whose predictions are combined using weights estimated from validation data. 
Compared with ensemble designs that derive diversity primarily from sample or feature subsets, reweighting, or feature transformations, 
its central contribution is the coordinated design of training objectives and anchor geometries that target different 
prognostic and overlap structures. Following validation, each expert is refitted on the combined training and validation data, 
retaining its distinctive structure while drawing on the full development sample. The shared architecture couples a structured 
outcome anchor with a learned neural correction. ACE serves as the reference estimator, while OW-ACE emphasizes well-supported 
covariate regions through overlap-weighted factual training. G$\Phi$-ACE constructs a global outcome-guided geometry from pooled 
prognostic associations, whereas A$\Phi$-ACE incorporates arm-specific outcome associations and penalties for treatment-group imbalance. 
O$\Phi$-ACE complements these approaches with an outcome-free, overlap-aware projection that replaces the anchor input with a lower-dimensional 
representation. GeoACE is evaluated against 11 comparators on eight benchmark protocols, alongside comparisons of weighting rules 
and aggregation controls. It achieves the lowest mean $\sqrt{\mathrm{PEHE}}$ and absolute average treatment effect error on all three 
IHDP protocols and closely approaches the best comparator on NEWS, although performance varies across the remaining benchmarks. 
The aggregation analyses show similar observed performance for inverse-DR weighting and equal averaging.

\end{abstract}


\section{Introduction}
\label{sec:introduction}

\subsection{Context and Motivation}
\label{sec:context-motivation}

Many scientific and decision-making problems require estimating how an outcome would change under an intervention, rather than merely predicting it from observed characteristics. This distinction is crucial in observational studies, where treatment--outcome associations may reflect both causal effects and systematic differences between treated and untreated units. Although randomized controlled trials reduce such differences, they may be costly, slow, unethical, or infeasible; observational data are more widely available, but treatment assignment is generally non-random \citep{rubin1974estimating,imbens2015causal,hernan2020causal,imbens2024causal}.

Population-average effects can also conceal clinically or operationally important variation in response. Heterogeneous treatment-effect (HTE) estimation therefore targets how effects vary with pre-treatment covariates, commonly through the conditional average treatment effect (CATE) \citep{athey2016recursive,wager2018estimation,kunzel2019metalearners,nie2021quasioracle}. Existing approaches include causal forests, meta-learners, and neural estimators such as TARNet, CFRNet, and DragonNet \citep{shalit2017estimating,shi2019adapting}. These methods demonstrate the value of causal inductive biases, but no single estimator is uniformly well matched to differences in overlap, prognostic structure, treatment imbalance, and sample size across datasets.

This observation suggests that the central challenge is not only improving the capacity of a single estimator, but also constructing a principled set of complementary causal views whose disagreement reflects meaningful differences in observational geometry.

This study addresses that variability by constructing complementary causal experts within a shared anchor--correction framework and combining their predictions using weights learned exclusively from internal validation predictions. Expert diversity is motivated by plausible data scenarios with differing prognostic structures, degrees of treatment overlap, and patterns of covariate imbalance. These scenarios guide the choice of inductive biases and covariate geometries, allowing each expert to emphasize particular aspects of the observational problem while retaining a common predictive backbone. Weight selection, expert refitting, and final evaluation are conducted as separate stages, ensuring that neither test outcomes nor counterfactual truth informs ensemble weight estimation.

\subsection{Problem Definition and Technical Challenges}
\label{sec:problem-definition}

Consider observations \(\mathcal{D}=\{(X_i,T_i,Y_i)\}_{i=1}^{n}\), where \(X_i\in\R^p\) contains pre-treatment covariates, \(T_i\in\{0,1\}\) is the treatment, and \(Y_i\) is the observed outcome. Under the potential-outcomes framework, each unit has potential outcomes \(Y_i(1)\) and \(Y_i(0)\), but only \(Y_i=T_iY_i(1)+(1-T_i)Y_i(0)\) is observed \citep{rubin1974estimating,imbens2015causal}. Because the unit-level contrast is never jointly observed, we target
\begin{equation}
    \tau(x)
    =
    \E\!\left[Y(1)-Y(0)\mid X=x\right].
    \label{eq:cate}
\end{equation}

Here, individualized prediction refers to estimating \(\tau(X_i)\), not observing the latent unit-specific effect. Identification relies on consistency, conditional exchangeability, and overlap \citep{hernan2020causal,imbens2015causal}. Under these assumptions, the treatment-specific outcome functions \(\mu_t(x)=\E[Y\mid T=t,X=x]\) identify the CATE as \(\mu_1(x)-\mu_0(x)\). 
The treatment assignment mechanism is summarized by the propensity score \(e(x)=\mathbb{P}(T=1\mid X=x)\), which expresses the probability of receiving treatment given covariates \(x\) and helps characterize the extent to which treated and untreated individuals share comparable covariate profiles \citep{rosenbaum1983central}.
These quantities underpin the estimators reviewed in \cref{sec:related-work}; their formal definitions and assumptions are stated in \cref{sec:problem-setup}.
Three finite-sample challenges motivate the proposed framework for heterogeneous treatment-effect estimation from observational data: treatment--control imbalance and limited overlap, variation in prognostic structure and sample size, and model selection without observed individual-effect labels. The first two motivate complementary expert geometries, while the third motivates validation-based aggregation \citep{shalit2017estimating,nie2021quasioracle,lan2024causalq}. The framework operates within the stated causal identification assumptions and does not address unmeasured confounding or interference \citep{hernan2020causal,imbens2015causal}.

\subsection{Contributions}
\label{sec:contributions}

The central contribution of GeoACE is a scenario-guided approach to designing complementary causal experts. The specific contributions are as follows:
\begin{itemize}[leftmargin=*]

    \item \textbf{Scenario-guided structural diversity.}
    GeoACE generates expert diversity through distinct inductive biases and anchor geometries, rather than by selecting different subsets of samples or features. The experts are designed to accommodate plausible observational data scenarios with differing prognostic structures, treatment overlap, and covariate imbalance. Each expert emphasizes particular aspects of these scenarios within the same causal identification assumptions, and all experts are refitted on the combined training and validation sample.

    \item \textbf{Complementary experts within a shared estimation framework.}
    Five experts implement this principle through a common anchor--correction architecture. ACE provides a reference estimator combining a structured outcome anchor with a learned neural correction; OW-ACE emphasizes regions with treatment overlap through weighted factual training; G$\Phi$-ACE captures pooled prognostic associations; and A$\Phi$-ACE incorporates arm-specific outcome associations and treatment-group imbalance penalties. The fifth expert, O$\Phi$-ACE, introduces an outcome-free, overlap-aware projection constructed from covariates and treatment assignment, replacing the anchor input with a lower-dimensional representation. This common architecture makes the role of these complementary modeling choices easier to examine.
	
	\item \textbf{Data-derived statistical priors for neural correction.}
Each expert constructs a compact statistical prior comprising treatment-specific anchor predictions, their contrast, and retained low-dimensional coordinates. Supplied alongside the learned covariate representation, this prior carries the expert's estimated outcome structure and geometry into the correction heads, linking the statistical anchor to flexible neural refinement.
	
    \item \textbf{Broad empirical evaluation with statistical assessment.}
GeoACE is evaluated against 11 comparators across eight benchmark protocols, alongside comparisons with individual experts, leave-one-expert-out analyses, and aggregation controls. Rank-based statistical tests, multiplicity-adjusted post hoc comparisons, and bootstrap uncertainty analyses supplement the reported performance metrics. Together, these evaluations provide a systematic empirical and statistical basis for examining expert complementarity, interpreting observed performance differences and their uncertainty, and identifying the benchmark settings in which the ensemble offers gains or exhibits limitations.

\end{itemize}


\section{Related Work}
\label{sec:related-work}

\subsection{Methodological Families and Their Trade-offs}
\label{sec:rw-taxonomy}

Methods for conditional average treatment-effect (CATE) estimation differ
primarily in how they infer the unobserved potential outcome and control the
resulting extrapolation. Classical outcome regressions encode heterogeneity
through treatment--covariate interactions and are readily interpretable, but
misspecified functional forms or omitted interactions can distort nonlinear
CATE surfaces \citep{hu2023heterogeneous}. Propensity-score matching,
subclassification, and weighting instead improve observed-arm comparability
\citep{rosenbaum1983central}; under suitable conditions weighting can be
efficient \citep{hirano2003efficient}, although extreme propensities amplify
variance and neither strategy removes dependence on measured confounding and
overlap.

Meta-learners organize ordinary prediction algorithms into causal estimators
\citep{kunzel2019metalearners,curth2021nonparametric}. The S-learner shares one
response model across arms and can be data-efficient, but may attenuate a weak
treatment signal. The T-learner permits distinct response surfaces, at the cost
of instability when an arm is small. The X-learner imputes arm-specific effects
and combines them, often using propensity information; it can exploit treatment
imbalance and a simple effect structure, but inherits first-stage outcome error.
Their modularity is useful, yet factual predictive accuracy alone does not
guarantee accurate counterfactual contrasts.

Pseudo-outcome and orthogonal learners align estimation more directly with the
effect. Doubly robust (DR) learners combine outcome and propensity estimates in
a corrected pseudo-outcome, while the R-learner residualizes both outcome and
treatment \citep{nie2021quasioracle,kennedy2023towards}. Orthogonal objectives
and cross-fitting can reduce first-order sensitivity to nuisance error
\citep{chernozhukov2018double}, but corrections become noisy under weak overlap
and remain dependent on identification assumptions. In particular, a DR
pseudo-outcome is an observable validation surrogate, not observed
individual-effect truth.

\subsection{Local, Bayesian, and Representation-Based Estimators}
\label{sec:rw-flexible}

Among flexible approaches to CATE estimation, tree-based methods provide a natural way to capture nonlinear and subgroup-specific treatment-effect heterogeneity. Causal trees partition covariate space by treatment-effect heterogeneity; honest sample splitting reduces adaptive bias, and forests stabilize the local estimates \citep{athey2016recursive,wager2018estimation}. Generalized Random Forests cast trees as adaptive neighborhood weights for local moment equations \citep{athey2019generalized}, whereas Orthogonal Random Forests and CausalForestDML combine localization with residualization or orthogonal scores \citep{oprescu2019orthogonal,chernozhukov2018double}. These methods capture nonlinear subgroup structure without prespecified interactions, but sparse neighborhoods and limited support can still destabilize local effects.

Bayesian approaches add posterior uncertainty to flexible response modeling.
BART estimates a regularized sum-of-trees response surface and contrasts its
two treatment evaluations \citep{chipman2010bart,hill2011bayesian}; BCF further
separates prognosis from treatment response and uses propensity information to
mitigate regularization-induced confounding \citep{hahn2020bayesian}. Multi-task
Gaussian processes share information between potential-outcome functions
through a joint covariance \citep{alaa2017bayesian}. Their uncertainty can
highlight data-poor regions, but remains conditional on the model and does not
protect against violations of exchangeability or positivity.

Among flexible CATE estimators, neural methods offer several ways to control
which information is shared, balanced, or separated across treatment arms.
TARNet uses a common representation with treatment-specific heads; CFRNet adds
MMD or Wasserstein balance, and SITE preserves local similarity among
informative, comparable observations
\citep{shalit2017estimating,yao2018site}. Balancing can improve cross-arm
comparability, but excessive balance may discard prognostic or effect-modifying
variation. DragonNet instead jointly estimates potential outcomes and
propensity and uses targeted regularization \citep{shi2019adapting}, while
disentangled architectures separate treatment- and outcome-related information
\citep{chu2021disentangled}. These alternatives reflect a genuine trade-off:
retaining treatment-predictive structure can impair alignment, whereas removing
it can erase outcome-relevant signal.

More specialized neural designs target treatment-effect heterogeneity through
their training objectives or information-sharing mechanisms. SubgroupTE
iteratively links soft response groups with outcome and propensity prediction
\citep{lee2025subgroupte}; HyperITE learns dynamic parameter sharing between
arm-specific predictors \citep{chauhan2024hyperite}; and PairNet trains on
cross-treatment outcome differences, with performance also depending on its
backbone and pairing rule \citep{nagalapatti2024pairnet}. These mechanisms can
better match heterogeneous response structure than a fixed shared/separate
architecture, but add estimation or optimization complexity and may be
unstable when the available sample is small.

\subsection{Latent and Prior-Fitted Models}
\label{sec:rw-generative}

Generative models infer missing responses or latent causal structure. CEVAE uses proxies and a variational confounder representation \citep{louizos2017cevae}; GANITE adversarially completes counterfactual outcomes before effect estimation \citep{yoon2018ganite}; and TEDVAE separates instrumental, risk, and confounding factors \citep{zhang2021tedvae}. Their flexibility is valuable when the assumed latent structure is supported, but the inferred counterfactual remains model-dependent, and validity depends on proxy quality, generative assumptions, and optimization stability. Prior-Fitted Networks amortize inference by learning from simulated tasks drawn from a specified pretraining prior \citep{muller2022transformers}, a principle that CausalPFN applies to CATE estimation through in-context transformer predictions without task-specific optimization \citep{balazadeh2025causalpfn}. This can make inference rapid, but accuracy depends on correspondence between the pretraining prior and the target data-generating process; it does not relax ignorability or overlap.

\subsection{CATE Model Selection and Aggregation}
\label{sec:rw-aggregation}

The diversity of CATE estimators introduces a further challenge: selecting or combining candidates when individual treatment effects are unavailable as validation labels. Orthogonal and doubly robust pseudo-outcome losses provide observable proxies for this purpose, although their variance can be substantial when overlap is weak or nuisance estimates are inaccurate. The R-learner supplies a related residualized criterion \citep{nie2021quasioracle}, while causal Q-aggregation combines candidate CATE functions under a doubly robust loss and provides oracle model-selection regret guarantees with higher-order nuisance-error terms \citep{lan2024causalq}. These results highlight that model selection and averaging require causal estimation criteria whose reliability depends on overlap and nuisance estimation.

Related approaches pursue stability and precision through different mechanisms: multi-task deep ensembles combine jointly trained causal 
networks to improve stability on complex inputs \citep{jiang2023multitask}, whereas inverse-variance weighting adapts pseudo-outcome 
regression to heterogeneous conditional precision \citep{fisher2024inversevariance}. Against this background, GeoACE examines how complementary 
experts designed around different overlap and outcome geometries can be combined within a common anchor--correction backbone. 
The evaluation considers uniform averaging, individual expert selection, and aggregation based on effect-oriented validation criteria under 
explicit information boundaries. 
To assess how the aggregation strategy affects performance, inverse-DR weighting is compared with equal averaging, single-expert selection, convex DR fitting, R-stacking, causal Q-aggregation, and DR ridge, with all comparisons based on the same fixed predictions from the five experts.

\subsection{Comparative Synthesis}
\label{sec:rw-comparative}

The strengths and limitations of these estimator families depend on the observational setting: shared models can borrow information across treatment groups but may obscure treatment-specific patterns, whereas separate models accommodate distinct responses but can lose efficiency when one group is small. Orthogonal corrections reduce sensitivity to nuisance estimation errors, although their variance can increase under weak overlap; similarly, balancing improves comparability, but overly restrictive representations may discard outcome-relevant information. Latent-variable, subgroup, and prior-fitted models offer additional flexibility, with performance depending on their structural assumptions or correspondence to the target distribution. These trade-offs motivate combining complementary experts within a controlled framework, allowing different approaches to information sharing, overlap, and outcome modeling to contribute without committing to a single inductive bias in advance. Such complementarity addresses finite-sample modeling challenges within the stated identification assumptions, but does not resolve unmeasured confounding or structural violations of positivity.

\paragraph{Positioning.}
GeoACE addresses variation across observational settings through a controlled library of complementary experts that share an anchor--correction backbone while encoding different modeling perspectives on treatment overlap and outcome structure. Their predictions are combined using a single set of convex weights estimated from internal validation predictions and fixed before test evaluation, allowing complementary perspectives to contribute at the task level rather than selecting one expert or assigning different experts to individual observations. This design aims to reduce model-selection uncertainty and finite-sample instability within the stated causal identification assumptions, although its validation losses remain imperfect proxies for treatment-effect error and its validity still depends on conditional exchangeability and positivity.

\section{Problem Setup}
\label{sec:problem-setup}

\subsection{Observed Data and Target Estimand}
\label{sec:observed-data-estimand}

Let $O_i=(X_i,T_i,Y_i)$, $i=1,\ldots,n$, denote independent
observations from a population of interest. Here, $X_i\in\mathcal{X}\subseteq
\R^p$ contains pre-treatment covariates, $T_i\in\{0,1\}$ is a binary
treatment indicator, and $Y_i$ is the factual outcome. Under the potential-
outcomes notation, each unit has two potential responses, $Y_i(0)$ and
$Y_i(1)$, but only
\begin{equation}
Y_i=T_iY_i(1)+(1-T_i)Y_i(0)
\label{eq:consistency-observed-outcome}
\end{equation}
is observed. We write
\begin{equation}
\mu_t(x)=\E\{Y(t)\mid X=x\},\qquad t\in\{0,1\},
\label{eq:conditional-response-functions}
\end{equation}
and define the conditional average treatment effect (CATE) as
\begin{equation}
\tau(x)=\E\{Y(1)-Y(0)\mid X=x\}
       =\mu_1(x)-\mu_0(x).
\label{eq:cate-target}
\end{equation}
The corresponding average treatment effect is
$\operatorname{ATE}=\E\{\tau(X)\}$. The principal target of this work is the
function $\tau(\cdot)$; ATE, ATT, potential-outcome error, and policy-oriented
quantities are treated as complementary evaluation criteria.

Our causal interpretation uses the standard identification conditions
\citep{rubin1974estimating,rosenbaum1983central,imbens2015causal}.
Specifically, we assume consistency, no interference between units,
conditional exchangeability,
\begin{equation}
\{Y(0),Y(1)\}\perp T\mid X,
\label{eq:conditional-exchangeability}
\end{equation}
and positivity,
\begin{equation}
0<e(x)=\Pp(T=1\mid X=x)<1
\label{eq:positivity}
\end{equation}
on the covariate region for which effects are estimated. Under these
conditions,
\begin{equation}
\tau(x)=\E(Y\mid X=x,T=1)-\E(Y\mid X=x,T=0).
\label{eq:identified-cate}
\end{equation}
These assumptions identify the estimand; they are not consequences of the
proposed learning procedure. In particular, neither representation balancing
nor ensemble weighting repairs unmeasured confounding or a structural absence
of one treatment arm.

\subsection{Data Partitioning and Information Boundaries}
\label{sec:data-information-boundary}

For each task, the supplied outer-training and outer-validation
partitions are combined to form a development sample.
Each observation is then assigned to one of three auxiliary
strata, \(E_i^{\mathrm{aux}}\in\{0,1,2\}\), through a seeded
randomization performed separately within each treatment arm.
This procedure distributes treated and control observations
approximately evenly across the strata, providing reproducible
groups for internal partitioning and stratum-adjusted anchor
estimation without treating the assigned labels as observed
sites or environments.

Within each treatment-by-stratum cell, observations are divided
into internal fitting, validation, and remaining subsets, with
nominal proportions of 60\%, 20\%, and 20\%, respectively,
subject to rounding and minimum cell-size constraints.
Accordingly,
\[
\mathcal D_{\mathrm{dev}} =
\mathcal D_{\mathrm{fit}} \cup
\mathcal D_{\mathrm{val}} \cup
\mathcal D_{\mathrm{rem}}.
\]

The remaining observations are included in development-wide
nuisance estimation and final refitting, but do not constitute
an additional validation or test set.

For \(s\in\{\mathrm{fit},\mathrm{val},\mathrm{rem}\}\), write
\(\mathcal D_s=\{(X_i,T_i,Y_i):i\in\mathcal I_s\}\).
Test outcomes and simulated counterfactual truth do not enter model fitting
or selector estimation. Observed test treatment assignments are used only
to construct the auxiliary test strata, as described in
\cref{sec:anchor-correction-backbone}. This treatment-aware prediction
boundary is distinct from using test outcomes for selection.

This separation is particularly important in CATE estimation because
individual treatment effects are not observed in standard validation data. For
a candidate estimator $\widehat{\tau}$, the ideal CATE risk
\begin{equation}
\mathcal{R}_{\tau}(\widehat{\tau})
=\E\left[{\widehat{\tau}(X)-\tau(X)}^2\right]
\label{eq:ideal-cate-risk}
\end{equation}
cannot be evaluated directly in the absence of counterfactual labels.
Accordingly, expert selection and ensemble weighting must be based on
observable proxy criteria. We use factual prediction error and a doubly robust
(DR) pseudo-outcome as complementary validation signals. Factual error provides
a comparatively stable measure of response-surface accuracy but is not
directly targeted at the treatment-effect contrast. In comparison, the DR
criterion is more closely aligned with CATE estimation, although it may exhibit
greater variability because it depends on propensity correction and estimated
nuisance functions.

\subsection{Expert Predictions and Convex Aggregation}
\label{sec:expert-ensemble-outputs}

Let \(K=5\) and let \(j\in\{1,\ldots,K\}\) index an expert. The
implemented prediction rule takes covariates and the auxiliary stratum label:
\begin{equation}
\widehat\mu_{0j}(x,E^{\mathrm{aux}}),\qquad
\widehat\mu_{1j}(x,E^{\mathrm{aux}}),\qquad
\widehat\tau_j(x,E^{\mathrm{aux}})
=\widehat\mu_{1j}(x,E^{\mathrm{aux}})
-\widehat\mu_{0j}(x,E^{\mathrm{aux}}).
\label{eq:expert-outputs}
\end{equation}

The same auxiliary label is retained for both potential-outcome
predictions, ensuring consistency in its use across the two
treatment conditions while preserving \(\tau(x)\) as the
population-level target. Because this label serves only as an
auxiliary component of the implemented estimator rather than
an assumed pre-treatment confounder, its dependence is left
implicit in the subsequent prediction notation. To maintain
comparability among experts within each task, all methods are
evaluated using the same prespecified set of model seeds, with
predictions averaged across five seeds for the six smaller
benchmarks. For the available ACIC runs, only a single seed
is used, making the corresponding averaging operation trivial.

To combine the experts, let $w=(w_1,\ldots,w_K)^\top$ denote a vector of
aggregation weights constrained to the probability simplex
\begin{equation}
\Delta_K
=
\left\{
w\in\mathbb{R}^K:
w_j\geq 0\ \text{for all }j,
\quad
\sum_{j=1}^{K}w_j=1
\right\}.
\label{eq:weight-simplex}
\end{equation}
Here, $w_j$ represents the contribution assigned to expert $j$. The
nonnegativity and unit-sum constraints in \cref{eq:weight-simplex} ensure that
the final estimator is a convex combination of the expert predictions.

Given $w\in\Delta_K$, the ensemble estimates the treatment-specific response
functions as
\begin{equation}
\widehat{\mu}_t^{\mathrm{ens}}(x;w)
=
\sum_{j=1}^{K}w_j\widehat{\mu}_{tj}(x),
\qquad t\in\{0,1\},
\label{eq:ensemble-mu}
\end{equation}
where $\widehat{\mu}_t^{\mathrm{ens}}(x;w)$ denotes the aggregated potential
outcome under treatment state $t$. The ensemble CATE estimate is then obtained
by contrasting the two aggregated response surfaces:
\begin{align}
\widehat{\tau}^{\mathrm{ens}}(x;w)
&=
\widehat{\mu}_1^{\mathrm{ens}}(x;w)
-
\widehat{\mu}_0^{\mathrm{ens}}(x;w)
\nonumber\\
&=
\sum_{j=1}^{K}w_j\widehat{\tau}_j(x).
\label{eq:ensemble-tau}
\end{align}
Thus, \cref{eq:ensemble-mu,eq:ensemble-tau} preserve internal coherence between
the ensemble potential-outcome predictions and their treatment-effect
contrast: aggregating the two response surfaces and then taking their
difference is algebraically equivalent to aggregating the expert-specific CATE
estimates directly.

A single weight vector $w$ is learned for each benchmark task or replication
and is subsequently held fixed across all individuals within that task. This
restriction provides task-level adaptation while avoiding the additional
estimation variance associated with learning a covariate-dependent gating
function from validation data for which individual causal effects are not
observed.


\section{Proposed Method}
\label{sec:method}

\subsection{Overview}
\label{sec:method-overview}

GeoACE is motivated by the observation that different observational
regimes may favor different inductive biases for heterogeneous treatment-effect
estimation. Rather than creating diversity by partitioning samples or
selecting subsets of features, GeoACE constructs five complementary causal
experts that represent different hypotheses about the underlying structure of
the observed data, including treatment overlap, prognostic relationships, and
covariate geometry. Their combination provides a way to account for uncertainty
about which structural assumption is most suitable for a given task.

All experts share an anchor--correction architecture. Each expert first
constructs a structured statistical anchor that provides an initial estimate of
the treatment-specific responses. A neural correction module, inspired by
representation-learning approaches such as TARNet, is then used to learn
departures from this anchor rather than estimating the responses from an
unstructured representation alone. In addition to covariate representations,
the prediction heads receive expert-specific prior information derived from the
anchor, including the estimated treatment-specific responses, their contrast,
and retained geometric coordinates. This design allows the neural component to
combine flexible nonlinear refinement with a statistically informed starting
point.

The five experts differ only in how they encode the observed-data assumptions
through their factual weighting schemes or anchor geometries. ACE provides the
reference anchor--correction model, OW-ACE emphasizes regions supported by
treatment overlap, G$\Phi$-ACE and A$\Phi$-ACE introduce outcome-guided
geometries based on global and arm-specific prognostic structures,
respectively, and O$\Phi$-ACE introduces an outcome-free overlap-aware
geometry. Thus, the expert library represents alternative views of the same
causal estimation problem rather than unrelated predictive models.

After expert fitting, GeoACE learns a task-level combination of the resulting
predictions using observable internal-validation predictions. The aggregation
step is intentionally separated from expert estimation: validation information
determines the combination rule, while the experts themselves remain unchanged.
This design provides a controlled alternative between selecting a single
estimator and learning a fully adaptive mixture of experts, avoiding the need
to learn observation-level gating functions from noisy causal surrogates.

\subsection{Common Anchor--Correction Backbone}
\label{sec:anchor-correction-backbone}

The common backbone separates a structured estimate of each treatment-specific
response from a regularized neural correction. The anchor captures outcome
variation along a compact set of prognostic and treatment-contrast directions,
while the correction can represent departures not adequately described by that
subspace. This division preserves a comparable estimation architecture across
experts without requiring their anchor geometries to coincide.

A compact synthetic illustration of ridge-based projection, nonlinear
anchoring, and residual correction is provided in Section~S2 of the
supplementary material, together with treatment-effect surface plots.

Let \(x^{\mathrm{pre}}\) denote the benchmark-adapted covariates before neural
standardization and \(x^{\mathrm{NN}}\) the corresponding fixed, standardized
neural input. For expert \(j\), the anchor receives
\(u^{(j)}(x)\in\mathbb R^{1\times p_j}\), equal to \(x^{\mathrm{pre}}\) for ACE
and OW-ACE, or to \(\Phi^G(x^{\mathrm{pre}})\),
\(\Phi^A(x^{\mathrm{pre}})\), or \(\Phi^O(x^{\mathrm{pre}})\) for the geometry
experts. The latter maps are estimated from internal-fitting data as described
in \cref{sec:phi-geometry}, while the neural path retains the same
standardized covariate input.

Arm-specific linear ridge fits identify an orthonormal basis
\(B^{(j)}\in\mathbb R^{p_j\times d_j}\), with \(d_j\in\{1,2\}\),
for the estimated prognostic and treatment-contrast subspace.
Let \(\bar u^{(j)}\) denote the internal-fitting mean of
\(u^{(j)}\). The matrix \(P^{(j)}=B^{(j)}B^{(j)\top}\)
projects inputs onto this subspace in the original covariate
space, whereas \(B^{(j)}\) extracts their low-dimensional
coordinates. The projected input, anchor predictions, and
associated prior are

\begin{equation}
\begin{aligned}
\widetilde u^{(j)}(x)
&=\bar u^{(j)}+\{u^{(j)}(x)-\bar u^{(j)}\}P^{(j)},\\
a_t^{(j)}(x)
&=f_{t,\alpha}^{(j)}\!\left(\widetilde u^{(j)}(x),E^{\mathrm{aux}}\right),
\qquad t\in\{0,1\},\\
\bigl(z_1^{(j)}(x),z_2^{(j)}(x)\bigr)
&=\operatorname{Pad}_2\!\left(
\operatorname{Std}_{\mathrm{fit}}\!\left[
\{u^{(j)}(x)-\bar u^{(j)}\}B^{(j)}\right]\right),\\
r^{(j)}(x)
&=\left[a_0^{(j)}(x),a_1^{(j)}(x),
a_1^{(j)}(x)-a_0^{(j)}(x),z_1^{(j)}(x),z_2^{(j)}(x)\right].
\end{aligned}
\label{eq:structured-anchor-prior}
\end{equation}
Here, \(f_{t,\alpha}^{(j)}\) is an arm-specific nonlinear ridge predictor with
penalized covariate features and unpenalized auxiliary-stratum intercepts.
The retained coordinates are standardized using internal-fitting statistics;
\(\operatorname{Pad}_2\) supplies a zero second coordinate when \(d_j=1\).
The resulting prior conveys both the anchor's response estimates and its
low-dimensional structure to the neural correction, while the anchor itself
remains fixed during correction training. 
Further computational details of the anchor fit are provided
in the supplementary material. 
Although both potential-outcome predictions
use the same auxiliary label, its test-time assignment depends on observed
treatment, making the implemented prediction rule treatment-aware without
changing the population CATE target.

Suppressing the expert index and auxiliary argument, let
\(h_\theta(x^{\mathrm{NN}})\in\mathbb R^{d_h}\) denote the learned
representation shared by the two treatment heads within an expert; parameters
are not shared across experts. Each head receives this representation together
with the anchor prior and estimates an anchor-relative correction:
\begin{equation}
q_t(x)=g_{t,\theta}\!\left([h_\theta(x^{\mathrm{NN}}),r(x)]\right),
\qquad t\in\{0,1\}.
\label{eq:correction-heads}
\end{equation}
For continuous outcomes, the response estimate adds this correction on the
outcome scale,
\begin{equation}
\widehat\mu_t(x)=a_t(x)+q_t(x),
\label{eq:continuous-anchor-correction}
\end{equation}
whereas binary outcomes use a correction on the log-odds scale,
\begin{equation}
\widehat\mu_t(x)
=\operatorname{expit}\!\left\{\operatorname{logit}\{a_t(x)\}+q_t(x)\right\}.
\label{eq:binary-anchor-correction}
\end{equation}
For binary outcomes, anchor probabilities are clipped to \([0.01,0.99]\)
before entering the prior and the logit. With the final layer of each head
initialized at zero, predictions begin at the anchor (or its clipped binary
probability), so learning proceeds from a structured response estimate rather
than from scratch.

For observation \((X_i,T_i,Y_i)\), the factual prediction is
\(\widehat\mu_{T_i}(X_i)=T_i\widehat\mu_1(X_i)+(1-T_i)\widehat\mu_0(X_i)\).
ACE and the geometry experts use arm-frequency weights
\begin{equation}
\omega_i^{\mathrm{arm}}
=\frac{T_i}{2\widehat\pi_c}
 +\frac{1-T_i}{2(1-\widehat\pi_c)},\qquad
\widehat\pi=\frac{1}{|\mathcal I_{\mathrm{fit}}|}
\sum_{i\in\mathcal I_{\mathrm{fit}}}T_i,\qquad
\widehat\pi_c=\min\{0.97,\max(0.03,\widehat\pi)\},
\label{eq:arm-frequency-weight}
\end{equation}
which reduce the influence of treatment-arm prevalence on factual fitting.
OW-ACE instead uses the overlap-emphasis weights in
\cref{eq:overlap-training-weight}. For continuous outcomes, the minibatch
objective is
\begin{align}
\mathcal{L}_{\mathrm{ACE}}(\theta)
={}&\frac{1}{|\mathcal{B}|}\sum_{i\in\mathcal{B}}
\omega_i\{Y_i-\widehat\mu_{T_i}(X_i)\}^2 \nonumber\\
&+\lambda_a\frac{1}{|\mathcal{B}|}\sum_{i\in\mathcal{B}}
\{q_0(X_i)^2+q_1(X_i)^2\}
+\lambda_o\Omega_{\mathrm{out}}(\theta).
\label{eq:ace-training-objective}
\end{align}
Here, \(\mathcal B\) is a minibatch, \(\omega_i\) is the expert-specific
weight, and \(\lambda_a,\lambda_o\geq0\) control anchor-deviation and
correction-head regularization, respectively. Specifically,
\(\Omega_{\mathrm{out}}(\theta)=\tfrac12\sum_{t=0}^{1}
\sum_{\ell\in\mathcal L_t}\|W_{t,\ell}\|_F^2\) penalizes linear-layer weights
in the two correction heads, excluding biases and representation weights.
For binary outcomes, weighted factual cross-entropy replaces the squared-error
term, while both regularizers remain unchanged. Penalizing both corrections
restrains departures from the anchor even for each individual's unobserved
potential outcome, so flexible refinement must be supported by improvements
in factual fit rather than unconstrained extrapolation. In this way, the
backbone provides a common statistical reference while leaving each expert's
weighting or geometry to express its distinct inductive bias.

\subsection{Complementary Causal Experts}
\label{sec:causal-experts}

The five experts vary either the emphasis of factual training or the
covariate geometry used to construct the anchor. This controlled diversity
represents different assumptions about which information is most reliable in
finite observational samples, while preserving the common correction
architecture and the same CATE target. Their distinctions are summarized in
\cref{tab:compact-expert-comparison}.

\subsubsection{Reference Anchor--Correction Expert (ACE)}

ACE constructs its anchor directly from the benchmark-adapted covariates and
uses the arm-frequency weights in \cref{eq:arm-frequency-weight}. It provides a
reference against which the effects of changing the training emphasis or the
anchor geometry can be assessed, rather than a deliberately weakened
baseline. Its anchor and correction pathways are illustrated in
\Cref{fig:ace-ow-backbone}\subref{fig:ace-backbone}.

\subsubsection{Overlap-Weighted Anchor--Correction Expert (OW-ACE)}

OW-ACE retains the reference anchor and correction architecture but gives
more influence to observations whose estimated treatment-assignment scores
are closer to one half. Let \(\widehat e_i\) be a treatment-assignment score
estimated by five-fold out-of-fold prediction on the development sample and
clipped to \([0.03,0.97]\). These fits include the auxiliary-stratum indicators
and are distinct from the internal-fitting propensity model used to construct
the geometry maps. The factual-loss weights become
\begin{equation}
\omega_i^{\mathrm{ov}}
=\frac{\widehat e_i(1-\widehat e_i)}
       {|\mathcal B|^{-1}\sum_{\ell\in\mathcal B}
        \widehat e_\ell(1-\widehat e_\ell)}.
\label{eq:overlap-training-weight}
\end{equation}
The numerator is largest when \(\widehat e_i\) is near one half, while the
denominator maintains a mean minibatch weight of one. In contrast to classical
arm-specific overlap weights, which use \(1-\widehat e_i\) for treated units
and \(\widehat e_i\) for controls \citep{li2018balancing}, OW-ACE applies the
same symmetric factor to both arms. It therefore emphasizes estimated
overlap in a predictive loss without inheriting the exact balance guarantees
of classical overlap weighting. The change affects which observations guide
the correction, not the target estimand or the support available in the data.
\Cref{fig:ace-ow-backbone}\subref{fig:ow-training} depicts this distinction.

\begin{figure}[!tbp]
  \centering
  \begin{subfigure}{0.98\textwidth}
    \centering
    \includegraphics[width=\linewidth]{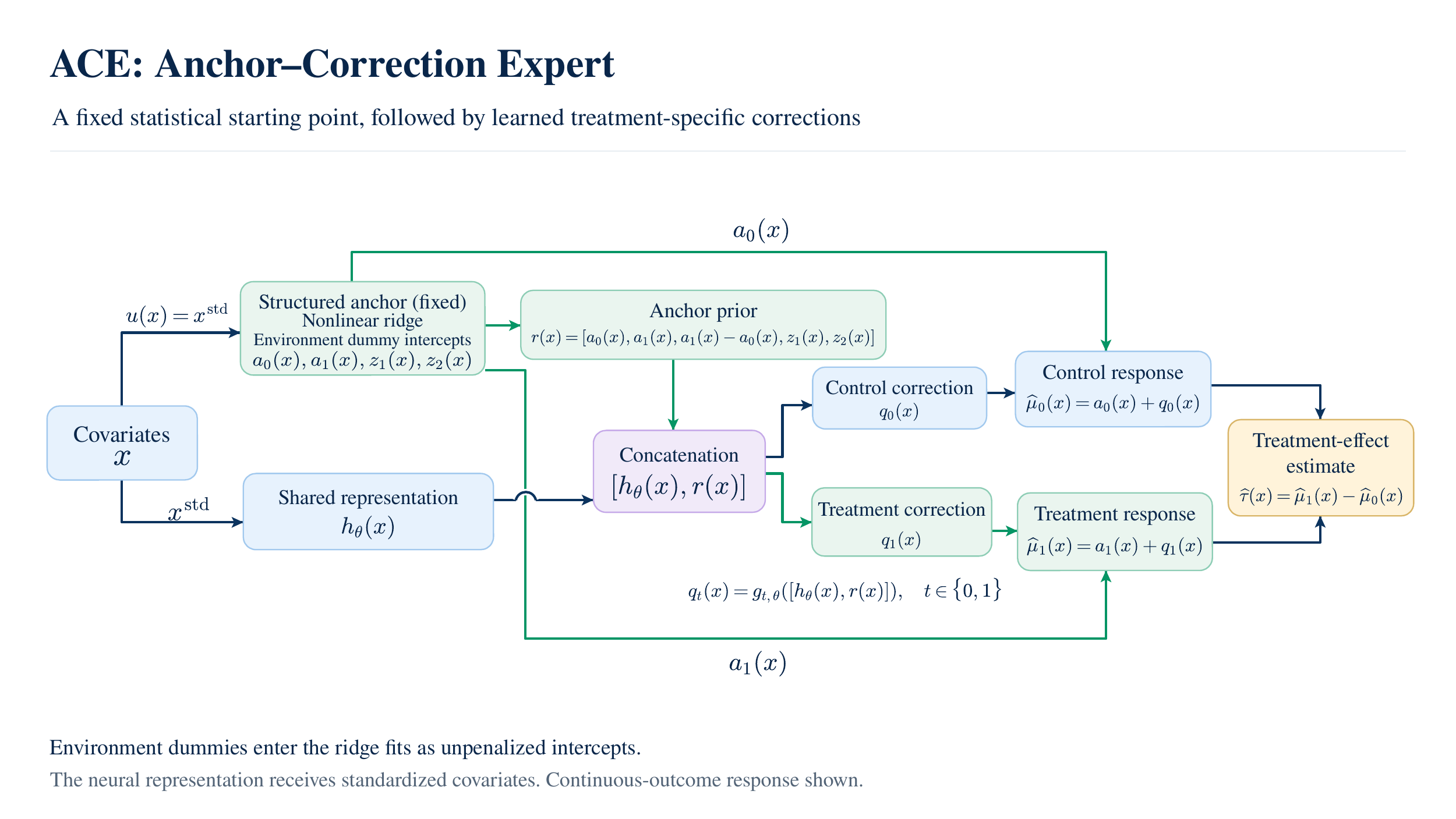}
    \caption{Reference ACE backbone.}
    \label{fig:ace-backbone}
  \end{subfigure}
  \par\medskip
  \begin{subfigure}{0.98\textwidth}
    \centering
    \includegraphics[width=\linewidth]{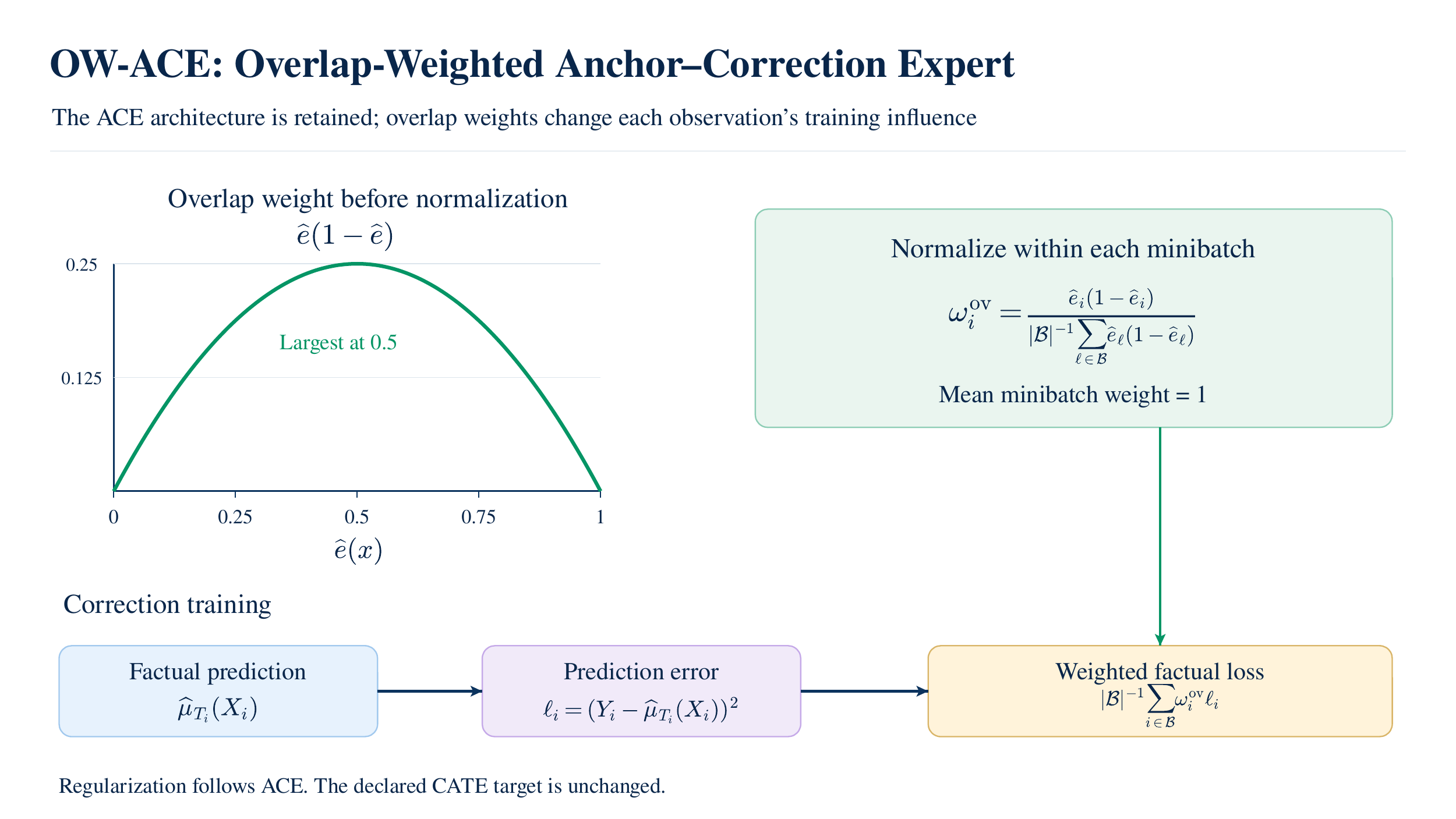}
    \caption{Overlap-emphasized factual loss in OW-ACE.}
    \label{fig:ow-training}
  \end{subfigure}
  \caption{Reference and overlap-weighted anchor--correction pathways.
  (a) The fixed anchor supplies initial responses and a prior to the neural
  correction heads. (b) OW-ACE changes the factual-loss weights without
  changing this architecture or its regularization. Continuous-outcome
  predictions are depicted; binary outcomes use
  \cref{eq:binary-anchor-correction}.}
  \label{fig:ace-ow-backbone}
\end{figure}

\subsubsection{Outcome-Guided and Outcome-Free Overlap Geometry}
\label{sec:phi-geometry}

Whereas OW-ACE changes the emphasis of factual fitting, the three geometry
experts change how covariate directions are represented to the structured anchor.
Their maps seek covariate variation in regions of estimated treatment overlap
while discouraging directions associated with measured treatment-group separation.
G$\Phi$-ACE and A$\Phi$-ACE also reward observed outcome association;
O$\Phi$-ACE deliberately omits that information. In every case, the neural
representation continues to receive \(X^{\mathrm{NN}}\), so the geometry
modifies the anchor rather than the correction network's covariate input.

Let \(X^{\mathrm{std}}\) denote the internally standardized covariates, with
\(p=q\) columns, and let \(S_{\mathrm{ov}},S_Y,S_B\in\mathbb R^{p\times p}\)
be symmetric positive semidefinite matrices representing overlap-weighted
covariate variation, outcome association, and measured treatment-group
imbalance. After trace normalization of their nonzero components, candidate
directions are ranked by
\begin{equation}
\mathcal{Q}(v)=
\frac{v^\top(S_{\mathrm{ov}}+S_Y+\eta I_p)v}
     {v^\top(S_B+\rho I_p)v},
\label{eq:phi-rayleigh-quotient}
\end{equation}
where \(\eta,\rho>0\) stabilize the two matrices. The corresponding
generalized eigenproblem is
\begin{equation}
\left(S_{\mathrm{ov}}+S_Y+\eta I_p\right)v
=\lambda\left(S_B+\rho I_p\right)v.
\label{eq:phi-generalized-eigenproblem}
\end{equation}
Thus, directions are preferred when overlap-weighted variation and, where
included, outcome association are large relative to measured imbalance.
The \(k=\lceil p/2\rceil\) leading generalized eigenvectors form the
expert-specific projection \(V_k\).

The outcome-guided experts differ in the evidence used to define this
trade-off. G$\Phi$-ACE uses a pooled, overlap-weighted covariate--outcome
association and penalizes treated--control mean separation after scaling by
within-arm covariance. Its projected scores are subsequently ZCA-whitened
to reduce linear correlation among the retained coordinates.
A$\Phi$-ACE instead uses outcome associations estimated separately in the
two treatment arms and penalizes differences in means, covariances, and the
propensity-predictive direction, without whitening. These alternatives
express, respectively, a shared prognostic geometry and one that can retain
associations specific to either arm.

Both outcome-guided experts augment the original standardized anchor input
with projected coordinates. For A$\Phi$-ACE, the map is
\begin{equation}
\Phi^A(X^{\mathrm{pre}})=
\left[
X^{\mathrm{std}},
\operatorname{Std}_{\mathrm{fit}}
\!\left\{X^{\mathrm{std}}V_k^A\right\}
\right],
\label{eq:phi-augmented-map}
\end{equation}
whereas G$\Phi$-ACE applies ZCA whitening and a final standardization to
its projected scores before the same concatenation. 
The ZCA transformation is specified in the supplementary material,
and both pathways are illustrated in
\cref{fig:outcome-guided-geometries}.

\begin{figure}[!tbp]
  \centering
  \begin{subfigure}{0.98\textwidth}
    \centering
    \includegraphics[width=\linewidth]{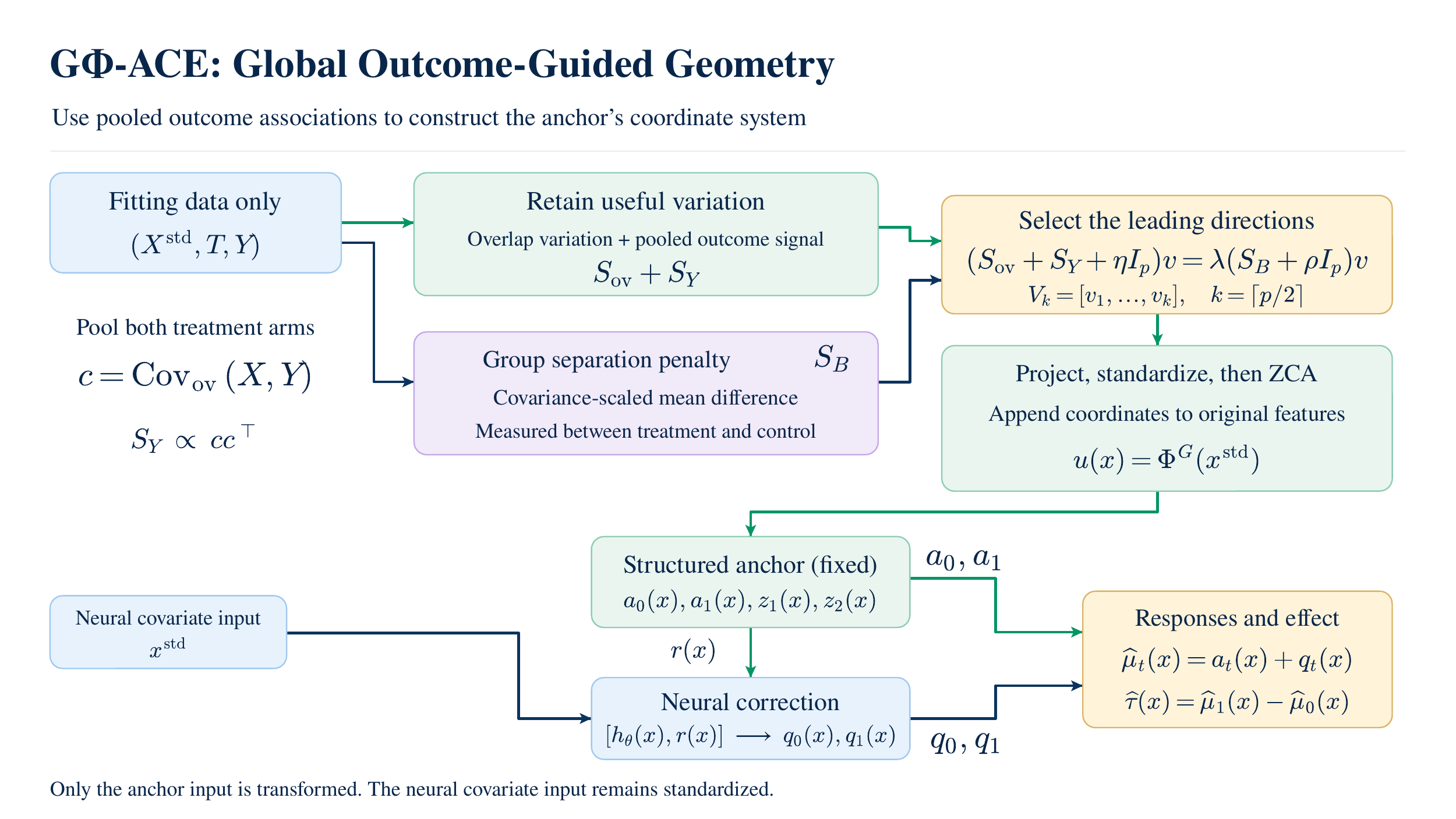}
    \caption{Global outcome-guided geometry, G$\Phi$-ACE.}
    \label{fig:gphi-pathway}
  \end{subfigure}
  \par\medskip
  \begin{subfigure}{0.98\textwidth}
    \centering
    \includegraphics[width=\linewidth]{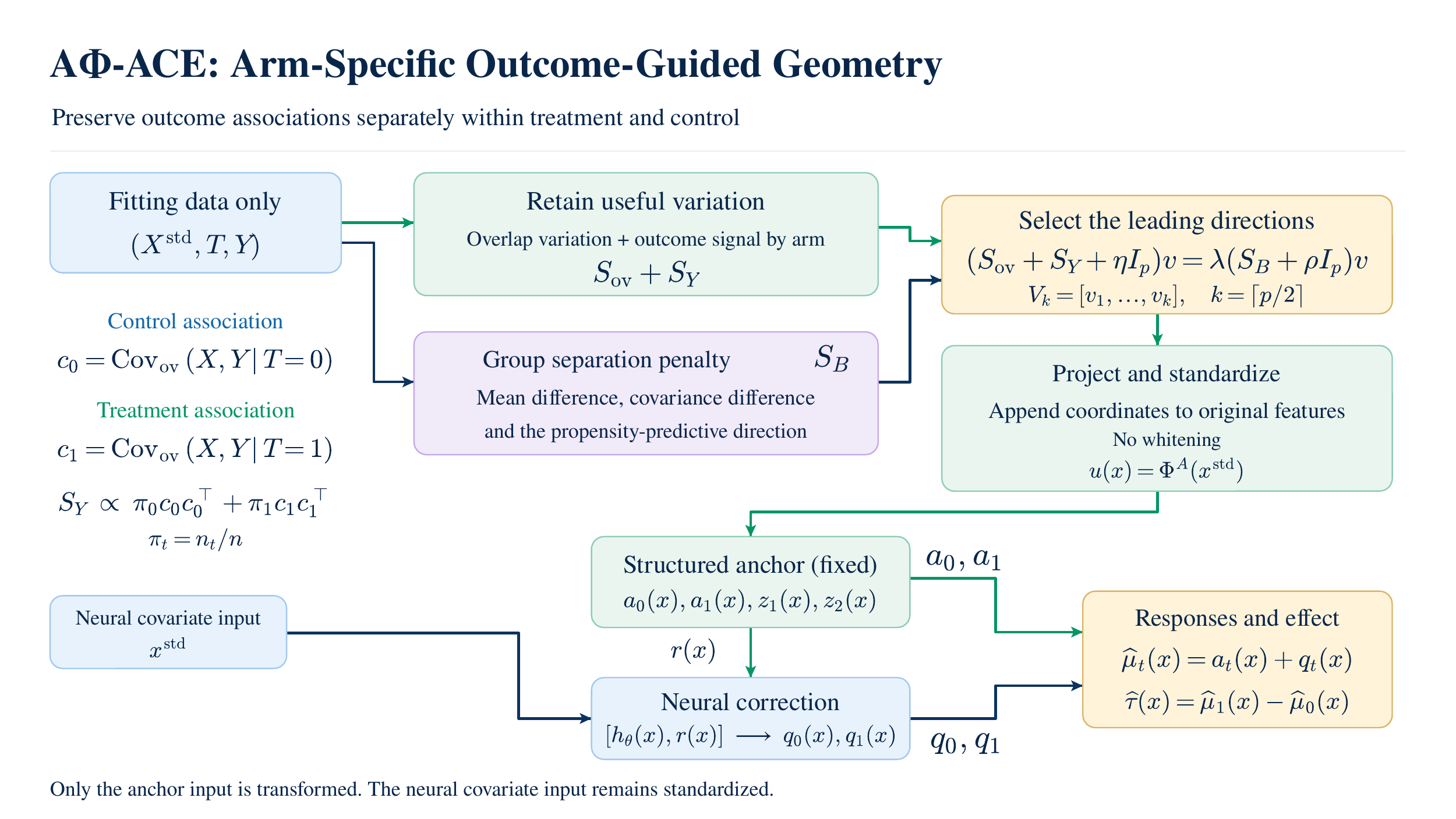}
    \caption{Arm-specific outcome-guided geometry, A$\Phi$-ACE.}
    \label{fig:aphi-pathway}
  \end{subfigure}
  \caption{Outcome-guided anchor geometries. (a) G$\Phi$-ACE uses pooled
  outcome association, a covariance-scaled mean penalty, and ZCA whitening.
  (b) A$\Phi$-ACE uses arm-specific outcome association and mean, covariance,
  and propensity-direction penalties without whitening. Both augment the
  anchor input while leaving the neural covariate input unchanged.}
  \label{fig:outcome-guided-geometries}
\end{figure}

O$\Phi$-ACE sets \(S_Y=0\), using overlap-weighted covariate variation
without outcome association to determine its projection. Its imbalance
matrix contains the same types of mean, covariance, and propensity-direction
penalties as A$\Phi$-ACE. Rather than appending projected scores, it replaces
the anchor input with
\begin{equation}
\Phi^O(X^{\mathrm{pre}})=
\operatorname{Std}_{\mathrm{fit}}
\!\left\{X^{\mathrm{std}}V_k^O\right\}.
\label{eq:ophi-replacement-map}
\end{equation}
This map is outcome-free, but the resulting anchor and neural correction
remain outcome-fitted. It therefore isolates an overlap-based geometric
assumption without withholding outcome information from response estimation,
as illustrated in \cref{fig:ophi-pathway}.

\begin{figure}[!tbp]
  \centering
  \includegraphics[width=0.98\textwidth]{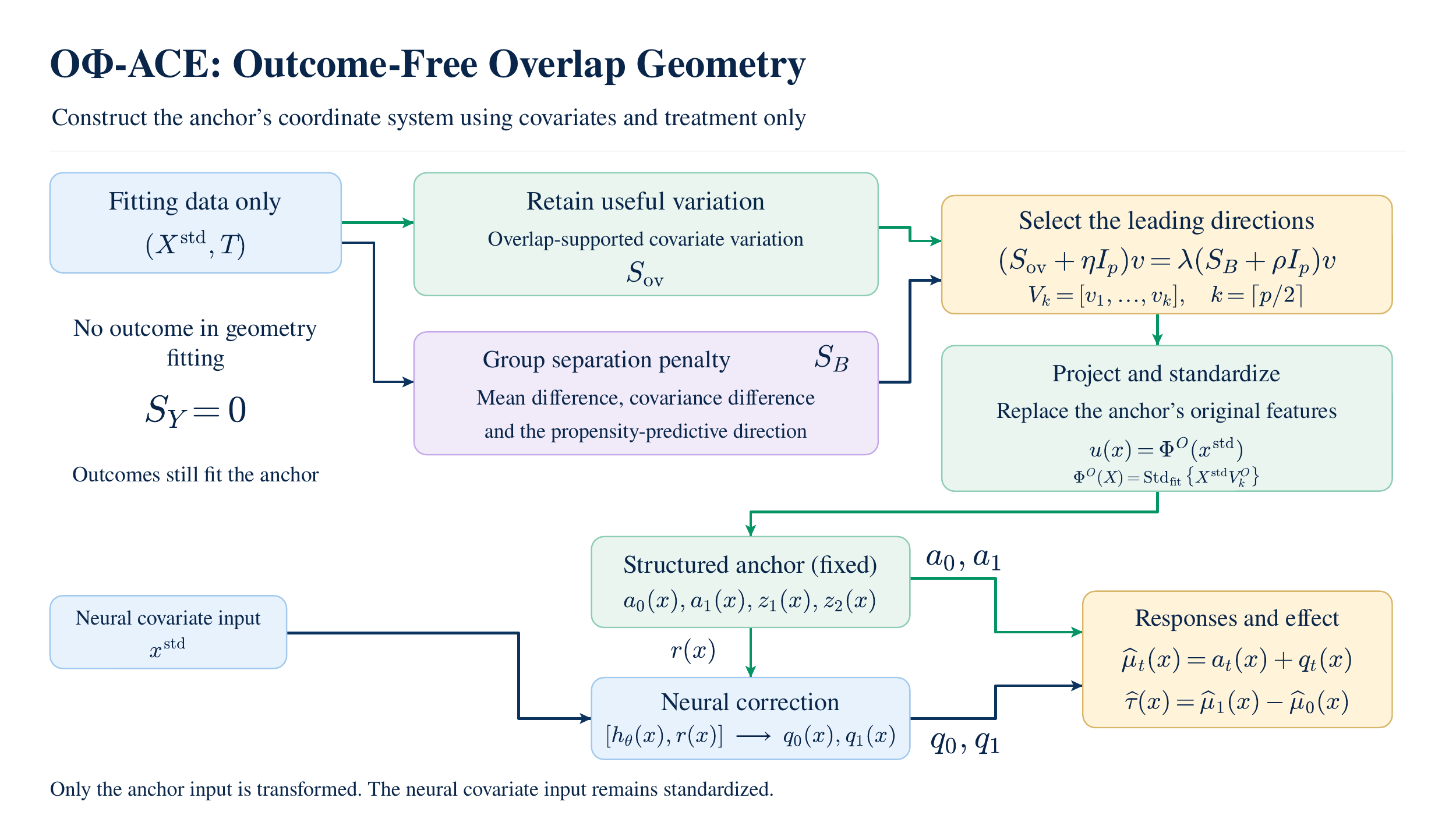}
  \caption{Outcome-free overlap geometry in O$\Phi$-ACE. Covariates and
  treatment assignment determine the projection, which replaces the anchor
  input; outcomes still inform the anchor and neural correction. The neural
  covariate input is unchanged. Continuous-outcome predictions are shown;
  binary outcomes follow \cref{eq:binary-anchor-correction}.}
  \label{fig:ophi-pathway}
\end{figure}

The geometry matrices are estimated exclusively from internal-fitting data
and retained during development-wide refitting and test prediction. 
The corresponding matrix estimators, normalization rules, and
whitening operations are specified in the supplementary material.
No validation or test outcomes or
counterfactual truth enter these maps, which cannot create common support or
remove bias from unmeasured confounding.

\begin{table*}[!tbp]
\centering
\caption{Comparison of the five causal experts. All use separately fitted
copies of the same anchor--correction backbone; geometry changes only the
anchor input, while the neural covariate input remains unchanged.}
\label{tab:compact-expert-comparison}
\scriptsize
\setlength{\tabcolsep}{3.2pt}
\renewcommand{\arraystretch}{1.22}
\begingroup
\renewcommand{\tabularxcolumn}[1]{>{\centering\arraybackslash}m{#1}}
\begin{tabularx}{\textwidth}{
|>{\centering\arraybackslash}m{1.35cm}
|>{\centering\arraybackslash}m{2.55cm}
|>{\centering\arraybackslash}m{2.55cm}
|>{\centering\arraybackslash}m{2.30cm}
|X|
}
\hline
\textbf{Expert} &
\shortstack[c]{\textbf{Anchor}\\\textbf{input}} &
\shortstack[c]{\textbf{Representation}\\\textbf{network input}} &
\shortstack[c]{\textbf{Factual-loss}\\\textbf{weighting}} &
\textbf{Distinctive mechanism} \\
\hline
ACE & $X^{\mathrm{pre}}$ & $X^{\mathrm{NN}}$ & Arm-frequency & Reference anchor geometry and training emphasis \\
\hline
OW-ACE & $X^{\mathrm{pre}}$ & $X^{\mathrm{NN}}$ & Overlap-emphasis & Symmetric overlap emphasis in the factual loss \\
\hline
G$\Phi$-ACE & $\Phi^G(X^{\mathrm{pre}})$ & $X^{\mathrm{NN}}$ & Arm-frequency & Pooled outcome association, covariance-scaled mean penalty, and ZCA whitening \\
\hline
A$\Phi$-ACE & $\Phi^A(X^{\mathrm{pre}})$ & $X^{\mathrm{NN}}$ & Arm-frequency & Arm-specific outcome association and mean, covariance, and propensity penalties \\
\hline
O$\Phi$-ACE & $\Phi^O(X^{\mathrm{pre}})$ & $X^{\mathrm{NN}}$ & Arm-frequency & Outcome-free overlap projection replacing the anchor input \\
\hline
\end{tabularx}
\endgroup
\end{table*}

\subsection{Validation-Based Ensemble Weighting}
\label{sec:validation-weight-learning}

To combine complementary experts without assuming that one is uniformly
reliable, GeoACE learns a task-level weight vector from internal-validation
predictions obtained before development-wide refitting.
For treatment state $t$, let $\widehat{\bm\mu}_{tj}^{\mathrm{val}}\in
\mathbb R^{n_v}$ collect expert $j$'s predictions on the
$n_v=|\mathcal I_{\mathrm{val}}|$ validation observations, and define
\begin{equation}
M_t=\left[\widehat{\bm\mu}_{t1}^{\mathrm{val}},\ldots,
          \widehat{\bm\mu}_{tK}^{\mathrm{val}}\right]
\in\mathbb R^{n_v\times K},\qquad H=M_1-M_0.
\label{eq:validation-prediction-matrices}
\end{equation}
Each column of $H$ is an expert's validation CATE estimate. The three
prespecified rules assign weights in the simplex $\Delta_K$ of
\cref{eq:weight-simplex}, using only observed validation information;
the primary rule is defined below, 
with the two alternative selectors defined and evaluated
in the supplementary material.

\subsubsection{Inverse Doubly Robust Risk Weighting}

Because factual prediction accuracy need not reflect accuracy of the
response contrast, the primary selector compares each expert with a common
doubly robust (DR) pseudo-outcome. Its outcome nuisance functions
$\widetilde\mu_t$ are ACE's selection-stage anchor predictions, before neural
correction, while $\widetilde e$ is the development-wide five-fold
out-of-fold propensity estimate, stratified by treatment and auxiliary
stratum. The same nuisance estimates are used for every expert, avoiding
expert-dependent validation targets. For numerical stability, let
$\widetilde e_c(x)=\min\{1-c,\max(c,\widetilde e(x))\}$, with $c=0.025$;
the stored propensity scores are already clipped more tightly. The
validation pseudo-outcome is
\begin{align}
\psi_i^{\mathrm{DR}}
={}&\widetilde\mu_1(X_i)-\widetilde\mu_0(X_i)
+\frac{T_i\{Y_i-\widetilde\mu_1(X_i)\}}
       {\widetilde e_c(X_i)} \nonumber\\
&-\frac{(1-T_i)\{Y_i-\widetilde\mu_0(X_i)\}}
        {1-\widetilde e_c(X_i)}.
\label{eq:dr-pseudo-outcome}
\end{align}
The outcome contrast is adjusted by inverse-propensity-weighted factual
residuals, yielding an observable surrogate for treatment effects rather
than individual-effect labels. The discrepancy of expert $j$ from this
shared surrogate is measured by
\begin{equation}
R_j^{\mathrm{DR}}
=\left[
\frac{1}{n_v}\sum_{i\in\mathcal I_{\mathrm{val}}}
\{\psi_i^{\mathrm{DR}}-\widehat\tau_j(X_i)\}^2
\right]^{1/2},
\label{eq:inverse-dr-risk}
\end{equation}
and converted to a convex weight through
\begin{equation}
w_j^{\mathrm{DR}}=
\frac{(R_j^{\mathrm{DR}}+\epsilon)^{-a}}
     {\sum_{\ell=1}^{K}(R_\ell^{\mathrm{DR}}+\epsilon)^{-a}},
\qquad a=1,\quad\epsilon=10^{-8}.
\label{eq:inverse-dr-weights}
\end{equation}
This smooth inverse-risk rule gives greater influence to experts closer to
the DR surrogate without selecting a single winner. Its focus on the
predicted treatment contrast motivates its use as the primary selector,
although the surrogate remains sensitive to nuisance estimation and weak
overlap. The alternative factual-RMSE and simplex-constrained DR ridge rules
are assessed under the same fixed expert predictions, without selecting a
weighting rule from test performance.

\subsection{Selection, Refit, and Evaluation}
\label{sec:selection-refit-protocol}

Internal validation determines each expert's training duration and
provides pre-refit predictions for ensemble weighting. The final response
models are then fitted on the full development sample, including validation
observations, without using the learned weights as training inputs.
This separation of selection from refitting is summarized in
\cref{alg:vw-cee} and illustrated in \cref{fig:geoace-workflow}.

\begin{algorithm}[H]
\caption{GeoACE workflow for one task with inverse-DR weighting.}
\label{alg:vw-cee}
\small
\begin{tabularx}{\linewidth}{@{}r@{\hspace{0.6em}}X@{}}
\textbf{Input:} & Development data, test covariates and observed
test treatment assignments, five prespecified experts, and fixed seeds.\\[2pt]

\textbf{1.} & \textbf{Prepare:} Establish benchmark preprocessing,
auxiliary strata, and internal fitting and validation partitions;
estimate development-wide out-of-fold propensities.\\

\textbf{2.} & \textbf{Fit and validate:} Train each expert's anchor
and correction, including its geometry where applicable, on internal
fitting data. Select training durations using internal validation
and retain pre-refit predictions and ACE anchor nuisances.\\

\textbf{3.} & \textbf{Refit:} Preserve preprocessing, geometry,
anchor projection, and prior-coordinate scaling. Refit anchor
outcome models and newly initialized corrections on the full
development sample for the selected durations.\\

\textbf{4.} & \textbf{Learn weights:} Construct the common DR
pseudo-outcome and estimate simplex weights from saved pre-refit
validation predictions, independently of the refitted experts.\\

\textbf{5.} & \textbf{Predict and aggregate:} Obtain potential-outcome
predictions from refitted experts, using observed test treatment
assignments for auxiliary strata. Average across configured seeds
and combine predictions with the frozen weights.\\[2pt]

\textbf{Output:} & Fixed ensemble response and CATE predictions
for subsequent evaluation.\\
\end{tabularx}
\end{algorithm}

The use of pre-refit predictions for weighting avoids evaluating refitted
models on development observations that contributed to their fitting, but
validation observations also guide checkpoint selection, so this is not an
independent validation guarantee. Test outcomes and available counterfactual
truth are reserved for evaluation, whereas observed test treatment
assignments enter prediction through the auxiliary-stratum procedure in
\cref{sec:data-information-boundary}.

\begin{figure}[H]
  \centering
  \includegraphics[width=0.98\textwidth]{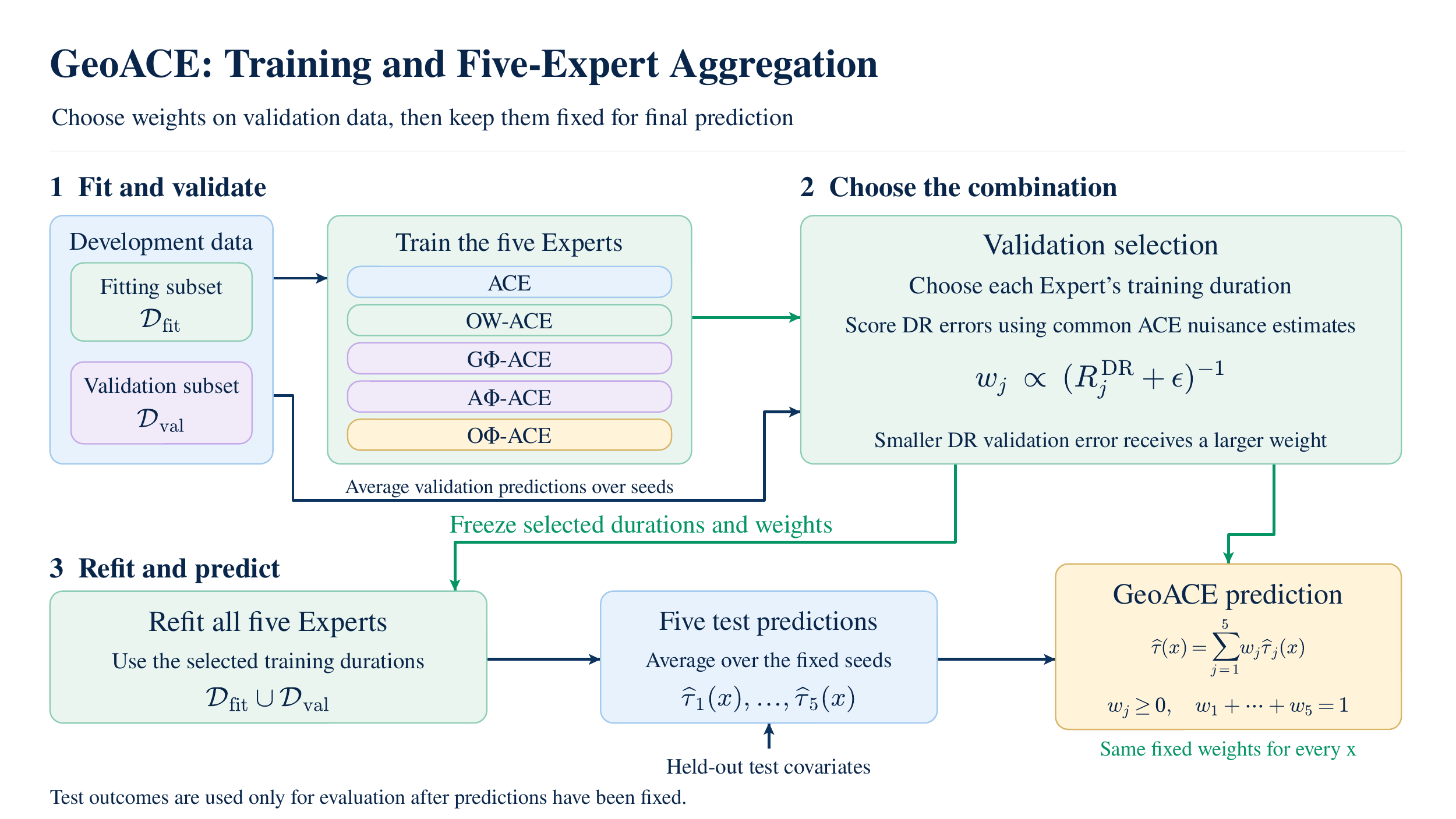}
  \caption{GeoACE fitting, refitting, and validation-based aggregation.
  Training durations and task-level ensemble weights are determined from
  internal-validation information, while the final predictions use refitted
  experts and fixed weights. The displayed $\widehat\tau$ denotes the
  ensemble contrast in \cref{eq:ensemble-tau}; the diagram conveys
  information flow rather than the execution order of independent steps.}
  \label{fig:geoace-workflow}
\end{figure}

The population interpretation of the DR surrogate and the limits of convex
aggregation, including the absence of a guarantee against the best expert,
are developed in \cref{app:method-properties}.

\section{Experimental Setup}
\label{sec:experiments}

\subsection{Datasets}
\label{sec:datasets}

The evaluation covers eight fixed benchmark protocols
(\cref{tab:benchmark-summary}) spanning outcome types, treatment
prevalences, covariate dimensions, and overlap conditions. Their evaluation
targets differ: the semi-synthetic protocols provide simulated effect truth,
TWINS uses paired mortality outcomes, and JOBS relies on experimental
comparisons without individual counterfactual labels.

The three IHDP protocols combine covariates from the Infant Health and
Development Program with simulated outcomes
\citep{hill2011bayesian,shalit2017estimating}. IHDPA and IHDPB are separate
100-replication constructions released with TEDVAE
\citep{zhang2021tedvae,zhang2021tedvaeCode}, while IHDP100 uses the
Johansson archive \citep{johansson2016learning,johanssonDataArchive}.
IHDPA preserves the released test set; IHDPB recombines the released
partitions before constructing the frozen 449/224/74 split, whereas
IHDP100 preserves its 75-observation test set. These constructions are
therefore evaluated as separate protocols.

NEWS combines real document covariates with simulated continuous responses
\citep{johansson2016learning}. Its raw representation contains 3,477 sparse
features; GeoACE uses a 100-component truncated-SVD representation whose
transformation and subsequent standardization are fitted on the outer-training
partition and then held fixed. TWINS uses the lighter and heavier twins'
binary mortality outcomes as paired potential outcomes, with synthetic
treatment assignment based on gestational-age category
\citep{almond2005costs,louizos2017cevae}. JOBS combines experimental and
observational job-training records \citep{lalonde1986evaluating} and is
assessed using ATT error and retrospective policy risk rather than PEHE.

The ACIC benchmarks combine real covariates with simulated outcomes.
ACIC2016 includes all 77 scenarios and five fixed replications per
scenario ($R01$, $R02$, $R05$, $R08$, and $R10$), yielding 385 tasks
\citep{dorie2019automated}. ACIC2017 covers additive i.i.d., additive
group-correlated, and non-additive error settings, each with eight designs
and 20 replications, yielding 480 tasks. The heteroskedastic error setting
is excluded because the released treatment-effect labels were computed
using an incorrect formula, as documented by the challenge authors
\citep{hahn2019acic2017}.

Both protocols retain all 58 original covariates
\citep{dorie2019automated,hahn2019acic2017}; the ACIC2017 data-generating
mechanisms involve eight of them. GeoACE's development-fitted imputation, categorical
encoding, and standardization yield 81--82 features for ACIC2016 and
84--85 for ACIC2017, depending on the observed categories. All methods share the benchmark tasks and
outer partitions, while preprocessing and model selection follow
method-specific protocols. Test outcomes and treatment-effect truth are
excluded from these transformations; detailed procedures appear in the
supplement.

\begin{table*}[!tbp]
\centering
\caption{Eight benchmark protocols. Sizes refer to the frozen outer
training, validation, and test partitions; GeoACE constructs its internal
partitions within their development union. Covariate dimensions are raw
feature counts, not necessarily the dimensions after method-specific
preprocessing.}
\label{tab:benchmark-summary}
\scriptsize
\setlength{\tabcolsep}{2.7pt}
\renewcommand{\arraystretch}{1.14}
\begin{tabularx}{\textwidth}{@{}lrrr
>{\centering\arraybackslash}p{2.45cm}
>{\raggedright\arraybackslash}p{2.65cm} X@{}}
\toprule
\textbf{Benchmark} & \textbf{$n$} & \textbf{$p_{\mathrm{raw}}$} &
\textbf{Tasks} & \textbf{Outer train/val/test} &
\textbf{Data source} & \textbf{Evaluation information} \\
\midrule
IHDPA & 747 & 25 & 100 & 449/224/74 &
TEDVAE archive \citep{zhang2021tedvaeCode} &
Continuous; conditional-mean truth \\
IHDPB & 747 & 25 & 100 & 449/224/74 &
TEDVAE archive \citep{zhang2021tedvaeCode} &
Continuous; conditional-mean truth \\
IHDP100 & 747 & 25 & 100 & 470/202/75 &
Johansson archive \citep{johanssonDataArchive} &
Continuous; conditional-mean truth \\
NEWS & 5,000 & 3,477 & 50 & 3,000/1,500/500 &
Johansson archive \citep{johanssonDataArchive} &
Continuous; conditional-mean truth \\
TWINS & 11,984 & 50 & 10 & 7,190/3,595/1,199 &
Twin-birth benchmark \citep{louizos2017cevae} &
Binary; paired potential outcomes \\
JOBS & 3,212 & 17 & 10 & 1,799/771/642 &
Johansson archive \citep{johanssonDataArchive} &
Binary; ATT and policy evaluation \\
ACIC2017 & 4,302 & 58 & 480 & 2,581/1,291/430 &
ACIC2017 challenge \citep{hahn2019acic2017} &
Continuous; conditional-effect truth \\
ACIC2016 & 4,802 & 58 & 385 & 2,881/1,441/480 &
ACIC2016 challenge \citep{dorie2019automated} &
Continuous; conditional-effect truth \\
\bottomrule
\end{tabularx}
\end{table*}

\subsection{Baselines}
\label{sec:baselines}

The baselines were selected to reflect both methodological diversity
and different perspectives on learning treatment effects from data.
Beyond their model classes, these methods differ in how they share
information across treatment arms, represent covariates, and capture
heterogeneity in treatment response.

The comparison includes ten estimator families, evaluated as eleven
variants (\cref{tab:baseline-publications}). The tree-based baselines
offer different ways of learning treatment effects: random-forest S-,
T-, and X-learners vary in how information is shared between treatment
arms, BART models the outcome surface using a Bayesian ensemble of
regression trees, and CausalForestDML (CF-DML) estimates effect
heterogeneity after residualizing treatment and outcome.

The neural baselines place greater emphasis on the learned
representation and training objective. TARNet combines a shared
representation with treatment-specific outcome heads, while CFRNet
introduces representation balancing through MMD or Wasserstein
penalties, evaluated here as separate variants. TEDVAE models
disentangled latent factors, SubgroupTE incorporates subgroup structure
into effect estimation, and PairNet learns from observed pairs drawn
from different treatment arms.

\begin{table}[H]
\centering
\caption{Comparator methods with publication and repository or software
references. Years and venues refer to the cited publications; the CF-DML
entry cites its generalized-random-forest foundation, and BART is dated
to its causal-inference application.}
\label{tab:baseline-publications}
\footnotesize
\setlength{\tabcolsep}{4pt}
\renewcommand{\arraystretch}{1.18}
\begin{tabularx}{\textwidth}{@{}
>{\raggedright\arraybackslash}p{2.45cm}
c
>{\raggedright\arraybackslash}X
>{\raggedright\arraybackslash}p{2.35cm}
c
>{\raggedright\arraybackslash}X@{}}
\toprule
\textbf{Method} & \textbf{Year} & \textbf{Publication / software} &
\textbf{Method} & \textbf{Year} & \textbf{Publication / software} \\
\midrule
SubgroupTE & 2025 &
ACM TIST \citep{lee2025subgroupte};
repository \citep{subgroupteCode} &
PairNet & 2024 &
ICML \citep{nagalapatti2024pairnet};
repository \citep{pairnetCode} \\
TEDVAE & 2021 &
AAAI \citep{zhang2021tedvae};
repository \citep{zhang2021tedvaeCode} &
S-learner & 2019 &
PNAS \citep{kunzel2019metalearners};
EconML \citep{econmlSoftware} \\
T-learner & 2019 &
PNAS \citep{kunzel2019metalearners};
EconML \citep{econmlSoftware} &
X-learner & 2019 &
PNAS \citep{kunzel2019metalearners};
EconML \citep{econmlSoftware} \\
CF-DML & 2019 &
Ann. Stat.\ \citep{athey2019generalized};
EconML \citep{econmlSoftware} &
TARNet & 2017 &
ICML \citep{shalit2017estimating};
repository \citep{cfrnetCode} \\
CFRNet-MMD & 2017 &
ICML \citep{shalit2017estimating};
repository \citep{cfrnetCode} &
CFRNet-WASS & 2017 &
ICML \citep{shalit2017estimating};
repository \citep{cfrnetCode} \\
BART (causal) & 2011 &
JCGS \citep{hill2011bayesian};
dbarts \citep{dbartsSoftware} & & & \\
\bottomrule
\end{tabularx}
\end{table}

All comparators are evaluated under the benchmark protocols in
\cref{tab:benchmark-summary}, using author-provided code, reference-based
implementations, or established software packages. They share task
identifiers and outer test partitions, while preprocessing, model selection,
and final fitting follow prespecified method-specific procedures.
Implementation provenance and protocol details are provided in the supplement.

\subsection{Evaluation Metrics}
\label{sec:metrics}

On benchmarks with reference effect labels, the primary metric is the
square root of the precision in estimation of heterogeneous effects
($\sqrt{\mathrm{PEHE}}$) \citep{hill2011bayesian,shalit2017estimating}:
\begin{equation}
\sqrt{\mathrm{PEHE}}
=\left[\frac{1}{n_{\mathrm{te}}}
\sum_{i\in\mathcal{I}_{\mathrm{te}}}
\left\{\widehat\tau(X_i)-\tau_i\right\}^{2}\right]^{1/2}.
\label{eq:metric-sqrt-pehe}
\end{equation}
Here, $\mathcal{I}_{\mathrm{te}}$ contains the $n_{\mathrm{te}}$ test
observations, $\widehat\tau(X_i)$ is the predicted effect, and $\tau_i$
is the benchmark reference: the conditional-mean contrast
$\mu_{1i}-\mu_{0i}$ for IHDP, NEWS, and ACIC, or the paired binary
contrast $Y_{1i}-Y_{0i}$ for TWINS. Taking the square root expresses
the error in outcome units, with smaller values indicating closer
agreement with the reference effects.

Average-effect accuracy is assessed through absolute ATE error
\citep{shalit2017estimating}:
\begin{equation}
\epsilon_{\mathrm{ATE}}
=\left|\frac{1}{n_{\mathrm{te}}}
\sum_{i\in\mathcal{I}_{\mathrm{te}}}
\left\{\widehat\tau(X_i)-\tau_i\right\}\right|.
\label{eq:metric-ate}
\end{equation}
The ATE criterion compares predicted and reference average effects in
the same held-out sample; because positive and negative estimation errors
can cancel, $\epsilon_{\mathrm{ATE}}$ is interpreted alongside
$\sqrt{\mathrm{PEHE}}$. Potential-outcome and factual RMSEs provide
secondary diagnostics where the corresponding reference outcomes are
available. For TWINS, factual Brier score \citep{brier1950verification}
and classification error at a fixed threshold of $0.5$ further describe
outcome-prediction quality.

Because individual counterfactual effects are unavailable for JOBS,
evaluation instead relies on absolute ATT error and policy risk,
following \citet{shalit2017estimating}. The predicted ATT is obtained
by averaging estimated effects over treated test units from the
randomized experiment and is compared with the observed outcome
difference between experimental treatment and control groups. Unlike
simulated effect references, this experimental contrast is subject to
sampling variability. Policy risk is evaluated on randomized test
units as $1-\widehat V(\widehat\pi)$, where
$\widehat\pi_i=\ind\{\widehat\tau(X_i)>0\}$ defines the treatment
rule. The policy-value estimator and its empty-cell convention are
detailed in the supplement. Because GeoACE uses observed treatment
labels to construct auxiliary strata, this evaluation reflects
retrospective benchmark performance rather than prospective policy
deployment. JOBS is therefore excluded from PEHE and ATE rank analyses.

For each held-out replication or ACIC task, metrics are calculated
before aggregation, with results reported as the mean and sample
standard deviation across tasks. Absolute ATE and ATT errors are
computed within each task before averaging, preserving the magnitude
of estimation errors regardless of their direction. The reported
SD thus captures variation across tasks rather than uncertainty
in the estimated mean and should not be interpreted as a standard
error or confidence interval.

\subsection{Implementation Details}
\label{sec:implementation}

GeoACE follows the benchmark preprocessing and information boundaries
described in \cref{sec:data-information-boundary,sec:datasets}.
For continuous outcomes, responses are standardized using
development-sample statistics during fitting, with predictions
transformed back to their original scale before evaluation.
Binary outcomes remain unstandardized, so the continuous training
objective is optimized on the normalized outcome scale.

All five experts use ELU activations and Adam optimization, with an
initial learning rate of \(10^{-3}\), a decay factor of 0.97 every
100 updates, minibatches of 100, gradient-norm clipping at 20,
and a maximum of 3,000 updates. Training duration is selected
by evaluating the full expert-specific validation objective,
including regularization, at initialization and every 100 updates.
The selected duration is then retained during refitting on the
development sample, with arm-frequency weights recalculated for
the corresponding fitting subset and regularization coefficients
fixed at \(\lambda_a=0.01\) and \(\lambda_o=10^{-4}\).
The experts share neural architecture settings within each benchmark,
although these settings vary across benchmarks. Five model seeds
(1701, 2903, 4307, 5711, and 7907) are used for the six non-ACIC
benchmarks, while the available ACIC2016 and ACIC2017 runs use
seed 1701. Ensemble weights are determined from saved pre-refit
validation predictions and remain independent of expert refitting.

To facilitate reproducibility and retrospective verification,
the recorded artifacts include run configurations, expert ordering,
selected training durations, ensemble weights, nuisance settings,
and predictions for the fitting, validation, and test partitions.
Together with the evaluation files, these records support checks
of the reported information boundaries and numerical results.
Additional implementation details, including comparator settings
and the reconciliation of JOBS evaluation conventions, are provided
in the supplement, 
while the computational construction of the anchor and geometry
matrices is documented in the supplementary material.


\section{Results}
\label{sec:results}

\subsection{Performance of the Individual Experts and Their Ensemble}

A central motivation for GeoACE is that different anchor geometries
may prove useful under different data-generating conditions.
\Cref{tab:expert-main-results} allows this premise to be examined
by comparing the five experts with their validation-weighted ensemble.
No individual construction dominates across the benchmarks:
O$\Phi$-ACE leads on ACIC2016, NEWS, TWINS, and JOBS;
OW-ACE on IHDP100 and IHDPA; A$\Phi$-ACE on ACIC2017;
and ACE on IHDPB. This changing ordering provides an empirical
motivation for combining the experts rather than relying on a single
geometry.

GeoACE achieves lower mean $\sqrt{\mathrm{PEHE}}$ than every constituent
expert on six of the seven benchmarks with reference individual effects.
The exception is TWINS, where O$\Phi$-ACE is marginally better
(0.3183 versus 0.3186), while the two are also closely matched in
JOBS policy risk. These results are consistent with complementary
expert predictions, although the benchmark means alone do not imply
that aggregation improves every individual task.

\begin{table*}[!tbp]
\centering
\caption{Mean test performance of the five individual experts
and GeoACE, reported as mean $\pm$ sample SD across held-out
tasks. Lower is better. The metric is $\sqrt{\mathrm{PEHE}}$
except for JOBS, where policy risk is reported.
Bold indicates the lowest mean in each row, while underlining
identifies the best-performing individual expert.}
\label{tab:expert-main-results}

\footnotesize
\setlength{\tabcolsep}{3pt}
\renewcommand{\arraystretch}{1.18}

\begin{tabular*}{\textwidth}{@{\extracolsep{\fill}}lcccccc@{}}
\toprule
\textbf{Benchmark} &
\textbf{ACE} &
\textbf{OW-ACE} &
\textbf{G$\Phi$-ACE} &
\textbf{A$\Phi$-ACE} &
\textbf{O$\Phi$-ACE} &
\textbf{GeoACE} \\
\midrule

ACIC2016
& 2.0232 $\pm$ 1.44
& 1.9706 $\pm$ 1.41
& 1.9904 $\pm$ 1.05
& 1.9742 $\pm$ 1.03
& \underline{1.9203} $\pm$ 1.00
& \textbf{1.7518} $\pm$ 1.19 \\

ACIC2017
& 0.8951 $\pm$ 0.48
& 0.8901 $\pm$ 0.49
& 0.8721 $\pm$ 0.47
& \underline{0.8692} $\pm$ 0.47
& 0.9527 $\pm$ 0.53
& \textbf{0.7990} $\pm$ 0.43 \\

IHDP100
& 0.7332 $\pm$ 0.55
& \underline{0.7042} $\pm$ 0.51
& 0.7106 $\pm$ 0.41
& 0.7104 $\pm$ 0.47
& 0.8139 $\pm$ 0.87
& \textbf{0.6742} $\pm$ 0.54 \\

IHDPA
& 0.5569 $\pm$ 0.18
& \underline{0.5469} $\pm$ 0.18
& 0.5558 $\pm$ 0.16
& 0.5671 $\pm$ 0.18
& 0.5858 $\pm$ 0.19
& \textbf{0.5195} $\pm$ 0.16 \\

IHDPB
& \underline{2.1948} $\pm$ 0.42
& 2.2271 $\pm$ 0.36
& 2.2215 $\pm$ 0.42
& 2.2209 $\pm$ 0.43
& 2.1968 $\pm$ 0.42
& \textbf{2.0963} $\pm$ 0.39 \\

NEWS
& 1.7690 $\pm$ 0.52
& 1.8055 $\pm$ 0.52
& 1.7614 $\pm$ 0.50
& 1.7428 $\pm$ 0.49
& \underline{1.6766} $\pm$ 0.48
& \textbf{1.6707} $\pm$ 0.47 \\

TWINS
& 0.3195 $\pm$ 0.01
& 0.3196 $\pm$ 0.01
& 0.3231 $\pm$ 0.02
& 0.3232 $\pm$ 0.02
& \textbf{\underline{0.3183}} $\pm$ 0.01
& 0.3186 $\pm$ 0.01 \\

JOBS$^\dagger$
& 0.24127 $\pm$ 0.06
& 0.25926 $\pm$ 0.06
& 0.24077 $\pm$ 0.06
& 0.24838 $\pm$ 0.06
& \underline{0.23762} $\pm$ 0.06
& \textbf{0.23758} $\pm$ 0.06 \\

\bottomrule
\end{tabular*}

\vspace{1mm}
\parbox{0.96\textwidth}{\scriptsize
$^\dagger$JOBS uses retrospective policy risk and is excluded
from $\sqrt{\mathrm{PEHE}}$ rank analyses.
SD describes variation across held-out tasks rather than
uncertainty in the estimated mean; JOBS test splits are not
independent experimental replications.}

\end{table*}

\subsection{Comparison with Established Estimators}

The comparison with established estimators asks a different question:
whether the gains over GeoACE's own experts carry over to methods built
on other modeling principles. The summary in
\cref{tab:benchmark-comparison}, with the complete comparison in
Appendix~\ref{app:additional-results}, shows that GeoACE leads the
three IHDP protocols and ranks second on NEWS, where its mean
$\sqrt{\mathrm{PEHE}}$ is only 0.0022 (0.13\%) above PairNet's.
By contrast, BART leads ACIC2016 and CF-DML leads ACIC2017 and TWINS.
Thus, the advantage of combining the proposed geometries is evident
within the expert library but does not extend uniformly to the broader
comparison; the small NEWS difference, in particular, should not be
interpreted as evidence of equivalence.

\begin{table}[!tbp]
\centering
\caption{GeoACE relative to the strongest comparator on each
$\sqrt{\mathrm{PEHE}}$ benchmark, reported as mean $\pm$
sample SD across held-out tasks. Lower values and ranks
are better. Bold indicates the lowest mean in each row;
ranks are among twelve estimators.}
\label{tab:benchmark-comparison}

\footnotesize
\setlength{\tabcolsep}{3pt}
\renewcommand{\arraystretch}{1.18}

\begin{tabular*}{\linewidth}{
@{\extracolsep{\fill}}lcccr@{}}
\toprule
\textbf{Benchmark} &
\textbf{GeoACE} &
\textbf{Best comparator} &
\textbf{Method} &
\textbf{Rank} \\
\midrule

ACIC2016
& 1.7518 $\pm$ 1.19
& \textbf{1.1418} $\pm$ 0.78
& BART & 8 \\

ACIC2017
& 0.7990 $\pm$ 0.43
& \textbf{0.4265} $\pm$ 0.29
& CF-DML & 6 \\

IHDP100
& \textbf{0.6742} $\pm$ 0.54
& 0.9883 $\pm$ 0.44
& CFRNet-MMD & \textbf{1} \\

IHDPA
& \textbf{0.5195} $\pm$ 0.16
& 0.6502 $\pm$ 0.51
& TEDVAE & \textbf{1} \\

IHDPB
& \textbf{2.0963} $\pm$ 0.39
& 2.2298 $\pm$ 0.38
& TEDVAE & \textbf{1} \\

NEWS
& 1.6707 $\pm$ 0.47
& \textbf{1.6685} $\pm$ 0.42
& PairNet & 2 \\

TWINS
& 0.3186 $\pm$ 0.01
& \textbf{0.3113} $\pm$ 0.01
& CF-DML & 7 \\

\bottomrule
\end{tabular*}
\end{table}

Across the seven $\sqrt{\mathrm{PEHE}}$ protocols, GeoACE has the
lowest observed mean rank (3.714) among the twelve methods, but the
Friedman comparison does not reject equal method ranks
($p=0.328$, Kendall's $W=0.162$), nor does its Iman--Davenport
correction ($p=0.330$)
\citep{demsar2006statistical,iman1980approximations}.
Because the analysis contains only seven protocols, including three
from the IHDP family, the favorable mean rank is best understood as
a descriptive summary of these benchmarks rather than evidence of
general superiority or equivalence.

\paragraph{Average-effect accuracy.}
Individual-effect recovery and average-effect calibration need not
follow the same pattern, motivating the complementary comparison of
absolute ATE errors in \cref{tab:average-effect-comparison}. GeoACE
again leads on IHDP100, IHDPA, and IHDPB, reducing mean absolute ATE
error by 40.4\%, 24.7\%, and 10.1\%, respectively, relative to the
strongest non-GeoACE estimator in each protocol. Its performance is
less competitive elsewhere, particularly on TWINS, illustrating why
ATE accuracy and heterogeneous-effect accuracy should be examined
separately. JOBS uses absolute ATT error and is considered below.

\begin{table*}[!tbp]
\centering
\caption{GeoACE relative to the strongest non-GeoACE comparator in
absolute average-effect error. Values are mean $\pm$ sample SD over
the same held-out tasks; lower is better. JOBS uses ATT error.
Ranks are among twelve estimators and are descriptive.}
\label{tab:average-effect-comparison}
\footnotesize
\setlength{\tabcolsep}{4pt}
\renewcommand{\arraystretch}{1.18}
\begin{tabular}{lccccc}
\toprule
Protocol & Metric & GeoACE & Best comparator & Method & Rank \\
\midrule
ACIC2016 & ATE
& 0.2181 $\pm$ 0.2619
& 0.0963 $\pm$ 0.1490 & BART & 5 \\

ACIC2017 & ATE
& 0.3082 $\pm$ 0.3535
& 0.1095 $\pm$ 0.1057 & CF-DML & 6 \\

IHDP100 & ATE
& \textbf{0.1137 $\pm$ 0.0882}
& 0.1909 $\pm$ 0.1375 & CFRNet-MMD & \textbf{1} \\

IHDPA & ATE
& \textbf{0.1042 $\pm$ 0.0810}
& 0.1384 $\pm$ 0.1163 & TEDVAE & \textbf{1} \\

IHDPB & ATE
& \textbf{0.2872 $\pm$ 0.2285}
& 0.3194 $\pm$ 0.2473 & TEDVAE & \textbf{1} \\

NEWS & ATE
& 0.1626 $\pm$ 0.1252
& 0.1207 $\pm$ 0.0870 & TEDVAE & 4 \\

TWINS & ATE
& 0.0239 $\pm$ 0.0090
& 0.0073 $\pm$ 0.0037 & S-learner & 9 \\

JOBS & ATT
& 0.0811 $\pm$ 0.0748
& 0.0633 $\pm$ 0.0769 & SubgroupTE & 8 \\
\bottomrule
\end{tabular}
\end{table*}

The ATE results place GeoACE second in mean rank (3.857), behind
TEDVAE (3.143). The exploratory Friedman test detects differences
among the twelve methods ($Q=23.505$, $p=0.0150$,
Kendall's $W=0.305$), with the Iman--Davenport correction
($p=0.0074$) and within-protocol permutation calibration
($p<0.01$) pointing in the same direction. An omnibus difference,
however, does not establish that GeoACE outperforms a particular
competitor: none of its eleven exact paired sign-swap rank contrasts
survives Holm adjustment. A normal approximation to mean-rank
differences identifies a contrast with CFRNet-WASS
(Holm-adjusted $p=0.0261$), whereas exact Wilcoxon comparisons of
protocol-mean errors do not. The dependence on the post-hoc procedure,
together with the small number of related protocols, calls for an
exploratory rather than a general-superiority interpretation.

\subsection{Retrospective Policy Evaluation on JOBS}

JOBS offers a complementary, experimentally anchored assessment in
which individual counterfactual effects are unavailable. Across the
ten held-out splits, GeoACE attains mean policy risk 0.23758
(corresponding policy value 0.76242) and absolute ATT error 0.08112,
ranking sixth and eighth, respectively, among the twelve methods.
The four-expert ensemble has nearly the same policy risk (0.23755),
indicating that the addition of O$\Phi$-ACE has little effect on this
particular endpoint. Since individuals recur across splits and the
auxiliary strata use observed test-treatment labels, these comparisons
characterize retrospective benchmark behavior rather than independent
experimental replications or prospective policy performance.


\section{Additional Analysis}
\label{sec:analysis}

\subsection{Contribution of \texorpdfstring{O$\Phi$-ACE}{O-Phi-ACE}}

O$\Phi$-ACE is of particular interest because its outcome-free,
overlap-aware projection introduces a representation distinct
from the outcome-informed geometries of the other experts.
Its competitive standalone performance across several benchmarks
further motivates examining whether this distinctive representation
provides complementary information when incorporated into the
ensemble. To investigate this question,
\cref{tab:four-five-ablation} compares the four-expert system
with the full five-expert GeoACE ensemble using the same retained
expert predictions and validation-based aggregation procedure.
The addition of O$\Phi$-ACE reduces mean
$\sqrt{\mathrm{PEHE}}$ across all seven applicable benchmarks,
with lower errors on 998 of 1,225 paired tasks.
The largest relative reductions occur on ACIC2016 and NEWS,
supporting a complementary contribution from this expert
without implying that it is uniformly superior in isolation
or beneficial for every task and evaluation metric.

\begin{table*}[!tbp]
\centering
\caption{Paired ablation of the fifth expert. $\Delta$ is five-expert minus
four-expert performance, so negative values favor the five-expert system.}
\label{tab:four-five-ablation}
\footnotesize
\setlength{\tabcolsep}{5pt}
\begin{tabular}{lrrrrr}
\toprule
\textbf{Benchmark} & \textbf{Four experts} & \textbf{Five experts} &
$\boldsymbol{\Delta}$ & \textbf{Relative change} & \textbf{Five-expert wins} \\
\midrule
ACIC2016 & 1.8188 & 1.7518 & -0.0670 & -3.69\% & 326/385 \\
ACIC2017 & 0.8196 & 0.7990 & -0.0206 & -2.51\% & 380/480 \\
IHDP100  & 0.6854 & 0.6742 & -0.0112 & -1.63\% & 83/100 \\
IHDPA    & 0.5351 & 0.5195 & -0.0156 & -2.92\% & 71/100 \\
IHDPB    & 2.1371 & 2.0963 & -0.0409 & -1.91\% & 81/100 \\
NEWS     & 1.7285 & 1.6707 & -0.0579 & -3.35\% & 49/50 \\
TWINS    & 0.3204 & 0.3186 & -0.0018 & -0.57\% & 8/10 \\
JOBS$^\dagger$ & 0.237553 & 0.237584 & +0.000031 & +0.01\% & 5/10 \\
\bottomrule
\end{tabular}
\end{table*}

The secondary metrics qualify this improvement. On TWINS and ACIC2017, the
fifth expert improves PEHE but can worsen mean-effect calibration, illustrating
that better recovery of heterogeneity need not imply better ATE calibration.
The primary ablation in \cref{tab:four-five-ablation} therefore concerns PEHE
(policy risk on JOBS); it does not establish gains on every secondary metric.

\subsection{Leave-One-Expert-Out Sensitivity}

While the preceding ablation examines the contribution of
O$\Phi$-ACE, a broader question is whether the ensemble
benefits consistently from each of its constituent experts.
To investigate this, each expert is removed in turn while
the remaining predictions are held fixed, and the
inverse-DR weights are renormalized using the same
validation risks. The ACE nuisance estimates defining
the DR pseudo-outcome remain unchanged even when ACE
itself is excluded, ensuring that each comparison
isolates the contribution of a prediction column under
a common aggregation rule.

The results reveal a clear distinction between consistent
and context-dependent contributions. Removing
O$\Phi$-ACE increases mean $\sqrt{\mathrm{PEHE}}$
across all seven applicable benchmarks, with relative
increases ranging from 0.57\% on TWINS to 3.83\% on
ACIC2016, whereas its contribution to JOBS policy risk
is negligible. OW-ACE also improves performance across
both ACIC and all three IHDP protocols, although its
removal benefits NEWS. The other experts exhibit
smaller and less consistent contributions, with some
removal conditions slightly improving ensemble accuracy.
Rather than requiring every expert to be individually
beneficial in every setting, these findings illustrate
how the value of an expert depends on the predictions
already represented in the ensemble. The complete
benchmark-specific pattern is shown in
\cref{fig:loeo-heatmap}, with complementary average-effect removal sensitivities
and a benchmark-balanced descriptive summary reported
in the supplementary material.

This complementarity is also metric-dependent.
Removing OW-ACE increases mean absolute ATE error
on both ACIC protocols and all three IHDP protocols,
whereas removing O$\Phi$-ACE reduces that error on
ACIC2017, IHDP100, and TWINS by 10.86\%, 4.27\%,
and 20.83\%, respectively, despite worsening it on
the remaining $\sqrt{\mathrm{PEHE}}$ benchmarks.
The corresponding change in JOBS ATT error is only
+0.96\%. These differences highlight that improvements
in recovering treatment-effect heterogeneity do not
necessarily translate into more accurate average-effect
estimation; the full average-effect ablation results
are reported in the supplementary material.

\subsection{Aggregation-Rule Controls}

An ensemble may benefit from the diversity of its constituent
predictions, the way those predictions are weighted, or both.
To examine the role of aggregation separately, the five expert
prediction columns are held fixed while their combination is
varied. Alongside the primary inverse-DR rule, the archived
controls include uniform averaging, single-expert selection,
convex DR fitting, R-stacking, causal Q-aggregation, and
DR ridge. The latter approaches draw on residualized
treatment-effect learning and causal model aggregation
\citep{nie2021quasioracle,lan2024causalq}. Although the
archived task-level losses and weights are available, the
complete fitting settings of these historical controls are
not, so the comparison characterizes the recorded
configurations rather than establishing general superiority
over the corresponding method families.

The results suggest that the choice of aggregation rule
matters, but not necessarily in favor of greater complexity.
Relative to inverse-DR, the benchmark-balanced excess loss
is 6.97\% for Best-DR, 3.69\% for Convex-DR, 3.30\% for
R-stacking, and 4.90\% for the archived Causal-Q configuration.
By contrast, Equal-5 and archived Ridge-DR yield small
differences of $-0.12$\% and $-0.09$\%, respectively, with
descriptive resampling intervals that include zero.
These findings are consistent with the value of combining
complementary experts, while offering no evidence that
inverse-DR weighting is uniquely preferable to uniform
averaging. The archived Ridge-DR configuration is distinct
from the prespecified simplex DR-ridge selector, and the
benchmark-balanced index, resampling procedure, and
provenance limitations are detailed in
\cref{app:aggregation-controls} and the supplementary material.

\subsection{Sensitivity of the Cross-Benchmark Ranking}

The cross-benchmark comparison includes TWINS, whose
reference treatment effects are derived from paired binary
potential outcomes rather than the conditional-mean contrasts
used in the other six benchmarks. This distinction motivates
a sensitivity analysis excluding TWINS to examine whether
the overall ranking depends on its particular evaluation
setting. GeoACE retains the lowest mean rank among the twelve
methods, improving from 3.714 across seven benchmarks to
3.167 across the remaining six. However, the Friedman test
($p=0.255$) and its Iman--Davenport correction ($p=0.250$)
remain non-significant
\citep{demsar2006statistical,iman1980approximations}.
The result therefore reinforces the descriptive consistency
of GeoACE's ranking without providing statistical evidence
of overall superiority.


\section{Discussion}
\label{sec:discussion}

The central finding is that differences among causal geometries can
be useful for heterogeneous treatment-effect estimation.
The five experts exchange rank across benchmarks, yet
GeoACE achieves lower mean $\sqrt{\mathrm{PEHE}}$ than each constituent on
six of the seven applicable protocols (\cref{tab:expert-main-results}).
The leave-one-expert-out analysis sharpens this interpretation:
removing O$\Phi$-ACE increases mean $\sqrt{\mathrm{PEHE}}$ on all seven
(\cref{fig:loeo-heatmap}), even though this expert is not invariably
the strongest standalone estimator. Its contribution therefore extends
beyond standalone accuracy. GeoACE leads on IHDP and closely follows
the NEWS leader, while established estimators remain stronger on
ACIC and TWINS.

Heterogeneous- and average-effect accuracy tell related but distinct
stories. GeoACE leads the three IHDP protocols in both comparisons and
ranks second on average across the seven ATE protocols
(\cref{tab:average-effect-comparison}), but gains in
$\sqrt{\mathrm{PEHE}}$ do not always translate into improved ATE calibration.
The lowest observed mean $\sqrt{\mathrm{PEHE}}$ rank is encouraging,
although the corresponding omnibus test is not significant. For ATE
error, the omnibus comparison detects differences among methods,
whereas the individual post-hoc contrasts involving GeoACE are
sensitive to the testing procedure. The evidence thus favors a
regime-dependent account of performance, rather than universal
superiority.

The aggregation controls clarify the source of this performance.
Equal weighting closely tracks the primary inverse-DR rule, suggesting
that much of the benefit comes from retaining complementary predictions
rather than from finely optimizing their coefficients. Several archived
selection and more flexible fitting controls perform less favorably on
the benchmark-balanced loss (\cref{tab:aggregation-bootstrap}), although
their incomplete optimization provenance precludes conclusions about
the underlying method families. The archived Ridge-DR control is
distinct from the fully specified simplex DR-ridge selector.
A low-dimensional selector can thus preserve expert diversity
without giving noisy validation signals excessive influence.

This interpretation remains conditional on the information available
to the estimator. Consistency, conditional exchangeability, and
positivity are required for the usual observational identification
argument, and neither expert diversity nor a doubly robust validation
surrogate removes the consequences of unmeasured confounding or
structural non-overlap. The auxiliary strata are reproducible
implementation labels, not observed causal environments. Because
their test-time construction uses observed treatment assignments,
the present evaluation establishes retrospective benchmark performance,
not a covariate-only rule for a new individual whose treatment has
yet to be assigned. A treatment-independent alternative must be
implemented and validated rather than presumed equivalent.

The computational and experimental boundaries also matter when
interpreting these results. Five expert fits, selected-duration
refits, and the associated geometry calculations require more
resources than a single estimator, even though the final task-level
simplex has only four free parameters and the expert fits can be
parallelized. Runtime and peak-memory measurements are needed to
quantify the trade-off. Likewise, eight protocols do not represent
eight independent domains: the IHDP variants share a dataset family,
ACIC tasks reuse covariates, and the ACIC2016 five-replication subset
was selected to reflect the aggregate difficulty of available
comparator results before GeoACE evaluation. The single ACIC seed
and method-specific fitting protocols also limit attribution of
performance differences.

A promising next step is to make the diversity of causal assumptions
more explicit. Prior-fitted approaches such as CausalPFN learn an
amortized mapping from observational data to causal-effect predictions
by pretraining on simulated tasks drawn from a prior over
data-generating processes \citep{balazadeh2025causalpfn}.
GeoACE instead fits complementary geometries to each benchmark task,
making the two approaches conceptually distinct but potentially
complementary. Future extensions could construct experts around
prespecified regimes of overlap, treatment assignment, outcome
complexity, and effect heterogeneity, then use observable diagnostics
to estimate which regimes are plausible for a new dataset. Such
regime scores could inform expert construction or a carefully
regularized, possibly regional, aggregation rule; they would not
by themselves identify the true generating mechanism. 
A complementary direction is to treat the matrix of expert effect
predictions as a learned prediction space and train a multi-output
regressor to map pretreatment covariates into that space. Using
out-of-fold expert predictions as training targets would allow the
regressor to learn relationships among the experts' responses, with
its predicted coordinates subsequently combined through the ensemble
aggregation rule. Such a surrogate could offer a compact covariate-only
prediction model, although its ability to preserve expert complementarity,
treatment-effect accuracy, and performance across overlap regimes would
require separate evaluation.
Weight
uncertainty, nonbinary treatments, and independent observational
validation with randomized reference evidence are further directions
for testing whether such adaptation improves on a fixed expert library.


\section{Conclusion}
\label{sec:conclusion}

GeoACE approaches heterogeneous treatment-effect estimation by preserving
several complementary causal inductive biases rather than committing to a
single representation. Its shared anchor--correction backbone makes the
experts comparable, while task-level validation aggregation combines their
predictions without altering the fitted experts. The outcome-free projection
of O$\Phi$-ACE adds a distinct source of geometric diversity, although its
outcome models remain outcome-trained.

Across the evaluated protocols, the strongest evidence is for the value of
this diversity: the ensemble improves on its constituent experts in most
$\sqrt{\mathrm{PEHE}}$ comparisons, and the ablations show that an expert's
contribution need not coincide with its standalone rank. The close results
for inverse-DR and equal weighting further suggest that a well-constructed
expert library can matter more than a highly adaptive aggregation rule.
The next challenge is to determine when each inductive bias is appropriate,
through treatment-independent prediction, independent-domain validation,
and eventually regime-aware expert selection. This direction retains the
central premise of GeoACE while opening a path toward causal estimators that
adapt their assumptions to the evidence available in a new setting.


\FloatBarrier


\bibliographystyle{plainnat}
\bibliography{references}

\clearpage
\appendix
\section*{Appendix}

\section{Mathematical Properties and Limits}
\label{app:method-properties}

GeoACE addresses estimation uncertainty through three complementary design
choices. The anchor--correction architecture gives each neural expert a
structured reference, leaving the correction to learn departures from it.
An informative anchor may reduce finite-sample variation, although errors
shared through that anchor can persist across experts. The $\Phi$ variants
explore alternative anchor geometries by favoring overlap-supported
directions and, where included, outcome association relative to measured
treatment-group imbalance (\cref{eq:phi-rayleigh-quotient}); these
projections neither eliminate treatment information nor create overlap
where it is absent. Task-level validation then combines these inductive
biases without assuming that any one geometry is universally preferable.

\subsection{Convex Aggregation}
\label{sec:why-convex-aggregation}

Using the same nonnegative, unit-sum weights for both treatment-specific
response surfaces makes aggregation consistent with their contrast
(\cref{eq:ensemble-mu,eq:ensemble-tau}) and yields a pointwise error bound.

\begin{proposition}[Convex CATE error bound]
\label{prop:convex-cate-bound}
For any $w\in\Delta_K$ and any $x$, with the auxiliary label held fixed
across expert predictions,
\begin{equation}
\left\{\widehat\tau^{\mathrm{ens}}(x;w)-\tau(x)\right\}^2
\leq
\sum_{j=1}^{K}w_j
\left\{\widehat\tau_j(x)-\tau(x)\right\}^2.
\label{eq:convex-cate-bound}
\end{equation}
Consequently, ensemble CATE risk is no greater than the corresponding
weighted average of expert risks.
\end{proposition}

\begin{proof}
Writing $e_j(x)=\widehat\tau_j(x)-\tau(x)$, Jensen's inequality gives
$\{\sum_j w_j e_j(x)\}^2\leq\sum_j w_j e_j(x)^2$ because the
weights are nonnegative and sum to one. Averaging over the evaluation
distribution proves the risk statement.
\end{proof}

This bound does not guarantee improvement over the best expert: the
benefit of aggregation depends on how expert errors interact, whereas
shared errors offer little opportunity for cancellation. Simplex weights
also keep the ensemble within the experts' pointwise range, preventing
extrapolation caused by large coefficients of opposite signs without
removing biases common to the experts.

\subsection{Validation Surrogates and Fixed Refit}
\label{sec:why-validation-surrogates}

Validation is necessarily indirect because individual counterfactual
outcomes are unavailable. Factual RMSE assesses observed-outcome fit,
but accurate prediction of the prognostic component need not yield an
accurate treatment contrast. A doubly robust (DR) pseudo-outcome is
more closely aligned with CATE, albeit dependent on nuisance estimation.
At the population level, without propensity clipping and under the
identification assumptions, candidate nuisance functions $m_t(x)$
and $p(x)$ satisfy
\begin{equation}
\E\{\psi^{\mathrm{DR}}\mid X=x\}=\tau(x)
\label{eq:dr-conditional-unbiasedness}
\end{equation}
if either $p(x)=e(x)$ or $m_t(x)=\mu_t(x)$ for both
$t\in\{0,1\}$, with $0<p(x)<1$. For a fixed covariate-only
prediction function $f$, this implies
\begin{equation}
\E\!\left[\{\psi^{\mathrm{DR}}-f(X)\}^2\mid X\right]
=
\Var(\psi^{\mathrm{DR}}\mid X)
+\{\tau(X)-f(X)\}^2.
\label{eq:dr-risk-decomposition}
\end{equation}
Since the variance term does not depend on $f$, population DR risk
targets CATE squared-error risk up to an additive constant
\citep{chernozhukov2018double,kennedy2023towards}. This motivates the
two DR-based selectors, with factual RMSE retained as a comparison
that avoids inverse-propensity residual corrections.

The population argument is not a finite-sample guarantee for the
implemented pipeline. Propensity estimation uses five-fold out-of-fold
predictions, but anchor outcome nuisances and checkpoint-selected
experts are not jointly cross-fitted with the selector; clipping,
limited overlap, and validation reuse further qualify its
interpretation. Moreover, auxiliary labels generated within observed
treatment arms enter prediction as estimator inputs rather than
measured pre-treatment covariates, so the covariate-only DR identity
does not itself validate this treatment-aware prediction rule.

GeoACE consequently restricts adaptation to one task-level simplex
vector---four free parameters for five experts---using smooth
inverse-DR-risk weights as the primary rule and a joint DR ridge
alternative that favors less concentrated weights. Both rely on saved
pre-refit validation predictions: training durations are selected
before refitting, and the learned weights do not alter expert
parameters. Because refitting can change the experts' errors, equal
weighting and individual experts remain important controls.

\subsection{Scope}
\label{sec:theory-scope}

These properties explain the design without establishing universal
superiority. Complementary expert errors may improve CATE estimation,
but aggregation cannot resolve unmeasured confounding, restore
unsupported counterfactuals, or guarantee improvement over the best
expert; fixed task-level weights also cannot adapt across covariate
regions. The empirical value of these choices must therefore be
assessed across repeated benchmark tasks rather than inferred from
the construction alone.

\section{Additional Results}
\label{app:additional-results}

\subsection{Leave-One-Expert-Out Results}

The leave-one-expert-out analysis evaluates the
sensitivity of the primary ensemble to each expert
under the fixed inverse-DR validation criterion.
For each removal, the remaining weights are
renormalized without retraining the experts or
changing the nuisance estimates. The heatmap in
\cref{fig:loeo-heatmap} summarizes the relative
increase in mean test loss following each removal,
with positive values indicating that retaining
the expert improves performance. 
Complementary average-effect removal sensitivities and a
benchmark-balanced descriptive summary are provided in the
supplementary material.

For benchmark $b$, the relative change in mean test loss
after removing expert $e$ is calculated as
\begin{equation}
\Delta^{\mathrm{LOEO}}_{e,b}(\%) =
100\,\frac{\bar L_{-e,b}-\bar L_{\mathrm{full},b}}
{\bar L_{\mathrm{full},b}},
\label{eq:loeo-relative-loss}
\end{equation}
where $\bar L_{-e,b}$ and $\bar L_{\mathrm{full},b}$
denote the mean test losses over the same held-out
tasks after removing expert $e$ and under the complete
five-expert ensemble, respectively. Positive values
favor retaining the expert, whereas negative values
favor its removal.

\begin{figure}[H]
\centering
\includegraphics[width=0.98\textwidth,height=0.72\textheight,keepaspectratio]{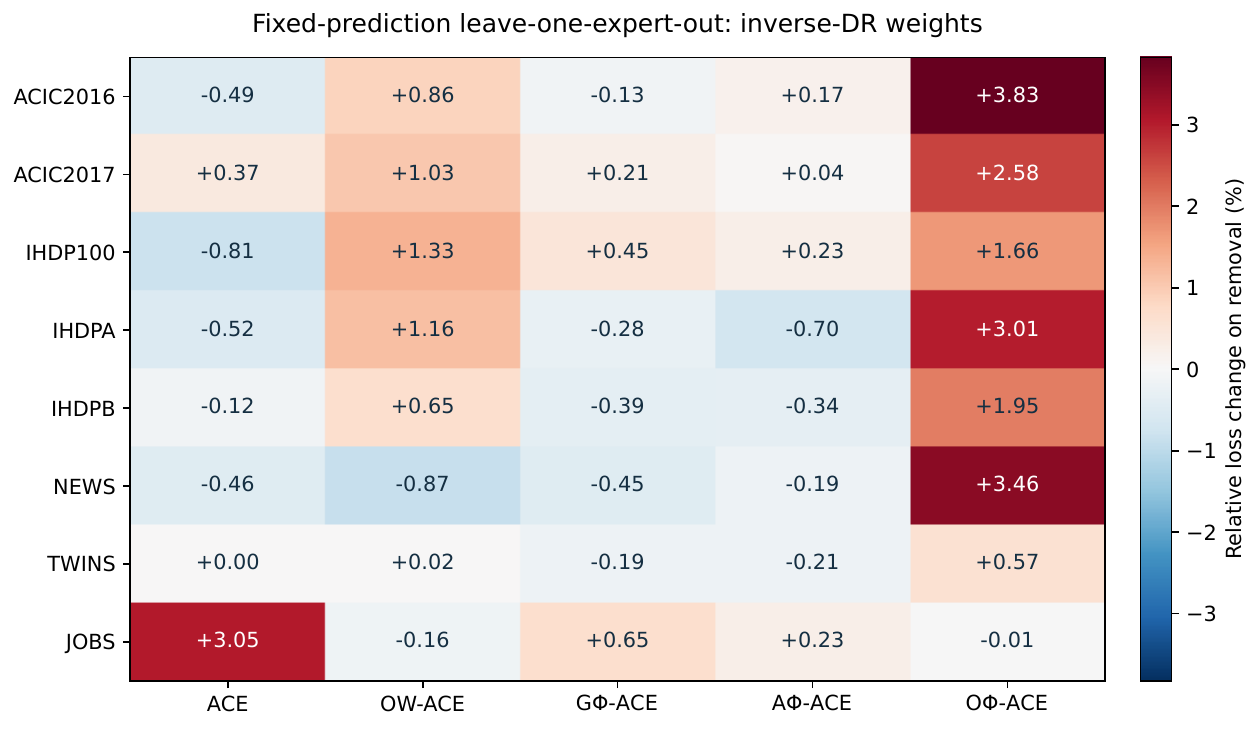}
\caption{Benchmark-specific leave-one-expert-out sensitivity under the primary
inverse-DR-risk selector. Each cell is the paired mean increase in the primary
loss after removing one expert, divided by the full-ensemble mean loss.
Positive values therefore favor retaining the expert. JOBS uses policy risk;
the remaining benchmarks use $\sqrt{\mathrm{PEHE}}$.}
\label{fig:loeo-heatmap}
\end{figure}

\subsection{Controls for the Aggregation Rule}

\label{app:aggregation-controls}

The aggregation-rule controls examine how alternative
weighting strategies combine the same five fixed expert
predictions. To compare their performance across benchmarks,
the relative change in mean test loss is summarized by

\begin{equation}
D_m =
\frac{100}{8}
\sum_{b=1}^{8}
\frac{\bar L_{m,b}-\bar L_{\mathrm{DR},b}}
     {\bar L_{\mathrm{DR},b}},
\end{equation}

where $\bar L_{m,b}$ and $\bar L_{\mathrm{DR},b}$
denote the mean test losses of aggregation rule $m$
and inverse-DR on benchmark $b$, respectively.
The loss is $\sqrt{\mathrm{PEHE}}$ for the seven
benchmarks with individual-effect ground truth
and policy risk for JOBS. Thus, $D_m$ measures
the benchmark-balanced percentage change rather
than an absolute difference on a common outcome
scale, with positive values favoring inverse-DR.
The descriptive 95\% percentile intervals are
obtained by resampling the eight benchmarks
with replacement (100,000 draws; seed 20261010).

\begin{table}[!htbp]
\centering
\caption{Archived-control benchmark-balanced relative loss versus inverse-DR. Positive values favor inverse-DR; intervals are descriptive.}
\label{tab:aggregation-bootstrap}
\footnotesize
\setlength{\tabcolsep}{3.5pt}
\begin{tabular}{lrr}
\toprule
Rule & Change (\%) & 95\% interval \\
\midrule
Best-DR & +6.97 & [+4.07, +9.61] \\
Causal-Q & +4.90 & [+2.28, +7.30] \\
Convex-DR & +3.69 & [+1.33, +5.92] \\
Equal-5 & -0.12 & [-0.36, +0.01] \\
R-stacking & +3.30 & [+0.79, +6.09] \\
Ridge-DR & -0.09 & [-0.31, +0.10] \\
\bottomrule
\end{tabular}
\end{table}

The benchmark-specific comparison in
\cref{fig:aggregation-controls} instead uses Equal-5
as its reference:
\begin{equation}
R^{\mathrm{Eq}}_{m,b}(\%) =
100\,\frac{\bar L_{m,b}-\bar L_{\mathrm{Eq},b}}
{\bar L_{\mathrm{Eq},b}},
\label{eq:aggregation-relative-equal5}
\end{equation}
where $\bar L_{m,b}$ and $\bar L_{\mathrm{Eq},b}$
denote the mean test losses of aggregation rule $m$
and uniform five-expert averaging on benchmark $b$,
respectively. Negative values favor the indicated
aggregation rule, while positive values favor Equal-5.
Unlike $D_m$, this quantity is calculated separately
for each benchmark rather than averaged across them.

\begin{figure}[H]
\centering
\includegraphics[width=0.98\textwidth,height=0.72\textheight,keepaspectratio]{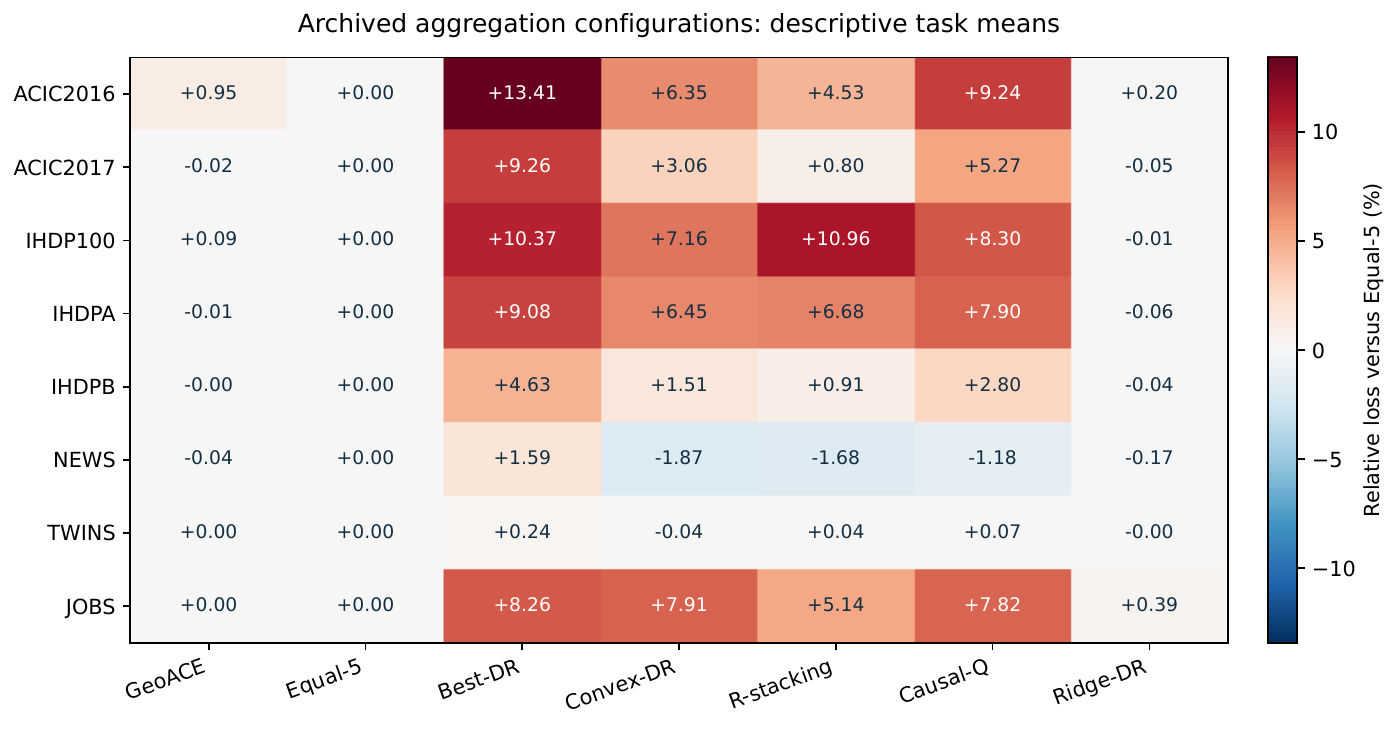}
\caption{Aggregation-rule controls using the same five frozen experts. Values
are benchmark-specific primary losses relative to Equal-5; negative values
favor the indicated rule. GeoACE, Equal-5, and Ridge-DR form a tight cluster,
whereas hard selection and several more flexible validation fits are less
stable across protocols.}
\label{fig:aggregation-controls}
\end{figure}

\Cref{tab:complete-benchmark-matrix} reports the complete test-set
$\sqrt{\mathrm{PEHE}}$ matrix underlying the rank analysis. The table retains
all methods rather than showing only per-benchmark leaders, making the strong
interaction between method and benchmark explicit.

\begin{table*}[!tbp]
\centering
\caption{Complete out-of-sample $\sqrt{\mathrm{PEHE}}$ comparison,
reported as mean $\pm$ sample SD across held-out tasks. Lower is
better; bold denotes the lowest mean for each benchmark. SD
summarizes task-to-task variation rather than uncertainty in the mean.}
\label{tab:complete-benchmark-matrix}
\footnotesize
\setlength{\tabcolsep}{1.2pt}
\begin{tabular*}{\textwidth}{@{\extracolsep{\fill}}lccccccc@{}}
\toprule
\textbf{Method} & \textbf{ACIC16} & \textbf{ACIC17} & \textbf{IHDP100} &
\textbf{IHDPA} & \textbf{IHDPB} & \textbf{NEWS} & \textbf{TWINS} \\
\midrule
BART & \textbf{1.1418}\,{\scriptsize$\pm$0.78} & 0.5174\,{\scriptsize$\pm$0.34} & 2.2985\,{\scriptsize$\pm$3.81} & 0.8929\,{\scriptsize$\pm$0.74} & 2.7530\,{\scriptsize$\pm$0.47} & 2.2304\,{\scriptsize$\pm$0.74} & 0.3117\,{\scriptsize$\pm$0.014} \\
CFRNet-MMD & 1.5669\,{\scriptsize$\pm$0.84} & 0.9673\,{\scriptsize$\pm$0.62} & 0.9883\,{\scriptsize$\pm$0.44} & 0.9183\,{\scriptsize$\pm$0.23} & 2.7432\,{\scriptsize$\pm$0.57} & 1.7400\,{\scriptsize$\pm$0.45} & 0.3184\,{\scriptsize$\pm$0.015} \\
CFRNet-WASS & 3.9804\,{\scriptsize$\pm$2.02} & 2.1141\,{\scriptsize$\pm$1.09} & 1.7064\,{\scriptsize$\pm$4.41} & 0.8430\,{\scriptsize$\pm$1.11} & 2.7050\,{\scriptsize$\pm$0.72} & 3.0557\,{\scriptsize$\pm$1.31} & 0.3246\,{\scriptsize$\pm$0.012} \\
CF-DML & 1.6658\,{\scriptsize$\pm$1.38} & \textbf{0.4265}\,{\scriptsize$\pm$0.29} & 3.8323\,{\scriptsize$\pm$6.14} & 1.3728\,{\scriptsize$\pm$1.33} & 3.4368\,{\scriptsize$\pm$0.54} & 2.4439\,{\scriptsize$\pm$1.00} & \textbf{0.3113}\,{\scriptsize$\pm$0.014} \\
PairNet & 3.1202\,{\scriptsize$\pm$1.42} & 1.5311\,{\scriptsize$\pm$0.95} & 2.0492\,{\scriptsize$\pm$3.51} & 0.8136\,{\scriptsize$\pm$0.69} & 2.4197\,{\scriptsize$\pm$0.49} & \textbf{1.6685}\,{\scriptsize$\pm$0.42} & 0.3292\,{\scriptsize$\pm$0.014} \\
S-learner & 1.5341\,{\scriptsize$\pm$0.83} & 0.6062\,{\scriptsize$\pm$0.32} & 3.3104\,{\scriptsize$\pm$5.81} & 1.1044\,{\scriptsize$\pm$1.11} & 3.2180\,{\scriptsize$\pm$0.56} & 1.9041\,{\scriptsize$\pm$0.60} & 0.3124\,{\scriptsize$\pm$0.014} \\
SubgroupTE & 2.1409\,{\scriptsize$\pm$0.93} & 1.0566\,{\scriptsize$\pm$0.58} & 1.5852\,{\scriptsize$\pm$2.08} & 0.8558\,{\scriptsize$\pm$0.50} & 3.2099\,{\scriptsize$\pm$0.50} & 1.7926\,{\scriptsize$\pm$0.45} & 0.3251\,{\scriptsize$\pm$0.014} \\
TEDVAE & 2.3805\,{\scriptsize$\pm$1.20} & 1.2346\,{\scriptsize$\pm$0.70} & 1.6359\,{\scriptsize$\pm$3.18} & 0.6502\,{\scriptsize$\pm$0.51} & 2.2298\,{\scriptsize$\pm$0.38} & 1.6729\,{\scriptsize$\pm$0.41} & 0.3138\,{\scriptsize$\pm$0.014} \\
TARNet & 1.6005\,{\scriptsize$\pm$0.85} & 0.8299\,{\scriptsize$\pm$0.45} & 1.2820\,{\scriptsize$\pm$1.31} & 1.0266\,{\scriptsize$\pm$0.26} & 2.8909\,{\scriptsize$\pm$0.54} & 1.7547\,{\scriptsize$\pm$0.47} & 0.3160\,{\scriptsize$\pm$0.015} \\
T-learner & 1.5031\,{\scriptsize$\pm$0.75} & 0.6666\,{\scriptsize$\pm$0.29} & 2.6787\,{\scriptsize$\pm$4.47} & 1.0034\,{\scriptsize$\pm$0.86} & 3.0653\,{\scriptsize$\pm$0.44} & 1.8613\,{\scriptsize$\pm$0.56} & 0.3198\,{\scriptsize$\pm$0.014} \\
X-learner & 1.5338\,{\scriptsize$\pm$0.74} & 0.6788\,{\scriptsize$\pm$0.29} & 2.6464\,{\scriptsize$\pm$4.44} & 1.0224\,{\scriptsize$\pm$0.88} & 3.0528\,{\scriptsize$\pm$0.42} & 1.8658\,{\scriptsize$\pm$0.57} & 0.3207\,{\scriptsize$\pm$0.014} \\
GeoACE & 1.7518\,{\scriptsize$\pm$1.19} & 0.7990\,{\scriptsize$\pm$0.43} & \textbf{0.6742}\,{\scriptsize$\pm$0.54} & \textbf{0.5195}\,{\scriptsize$\pm$0.16} & \textbf{2.0963}\,{\scriptsize$\pm$0.39} & 1.6707\,{\scriptsize$\pm$0.47} & 0.3186\,{\scriptsize$\pm$0.014} \\
\bottomrule
\end{tabular*}
\end{table*}

\setlength{\floatsep}{6pt}
\setlength{\textfloatsep}{8pt}
\begin{table}[!htbp]
\centering
\captionsetup{skip=6pt}
\renewcommand{\arraystretch}{1.0}

\begin{minipage}[t]{0.40\linewidth}
\captionsetup{type=table}
\centering
\caption{Mean PEHE ranks across seven protocols. Ordering is descriptive.}
\label{tab:average-ranks}
\footnotesize
\setlength{\tabcolsep}{3.5pt}
\begin{tabular}{lr}
\toprule
Method & Mean rank \\
\midrule
GeoACE & 3.714 \\
BART & 5.000 \\
TEDVAE & 5.143 \\
CFRNet-MMD & 5.286 \\
TARNet & 6.143 \\
T-learner & 6.857 \\
PairNet & 6.857 \\
X-learner & 7.286 \\
S-learner & 7.429 \\
SubgroupTE & 7.714 \\
CF-DML & 8.000 \\
CFRNet-WASS & 8.571 \\
\bottomrule
\end{tabular}
\end{minipage}
\hfill
\begin{minipage}[t]{0.57\linewidth}
\captionsetup{type=table}
\centering
\caption{JOBS retrospective policy risk on all ten test splits, including harmonized constant-policy cells. Entries are mean $\pm$ sample SD.}
\label{tab:jobs-complete}
\footnotesize
\setlength{\tabcolsep}{3.5pt}
\begin{tabular}{lcr}
\toprule
Method & Policy risk & $n$ \\
\midrule
SubgroupTE & 0.2247 $\pm$ 0.0622 & 10 \\
TARNet & 0.2291 $\pm$ 0.0454 & 10 \\
X-learner & 0.2293 $\pm$ 0.0393 & 10 \\
T-learner & 0.2311 $\pm$ 0.0367 & 10 \\
S-learner & 0.2374 $\pm$ 0.0442 & 10 \\
GeoACE & 0.2376 $\pm$ 0.0553 & 10 \\
PairNet & 0.2422 $\pm$ 0.0461 & 10 \\
CF-DML & 0.2425 $\pm$ 0.0509 & 10 \\
TEDVAE & 0.2476 $\pm$ 0.0366 & 10 \\
CFRNet-MMD & 0.2521 $\pm$ 0.0663 & 10 \\
BART & 0.2636 $\pm$ 0.0387 & 10 \\
CFRNet-WASS & 0.2638 $\pm$ 0.0530 & 10 \\
\bottomrule
\end{tabular}
\end{minipage}

\end{table}

\begin{table}[!htbp]
\centering
\captionsetup{skip=6pt}
\caption{Exploratory omnibus rank comparisons,
including the sensitivity analysis excluding TWINS.
Related benchmark protocols limit interpretation
across independent application domains.}
\label{tab:omnibus-rank-tests}
\footnotesize
\setlength{\tabcolsep}{3.5pt}
\begin{tabular}{lrrrrr}
\toprule
Metric & Geo rank & $Q$ & $p$ & $W$ & I--D $p$ \\
\midrule
$\sqrt{\mathrm{PEHE}}$
& 3.714 & 12.495 & 0.3276 & 0.162 & 0.3302 \\

$\sqrt{\mathrm{PEHE}}$ (excluding TWINS)
& 3.167 & 13.615 & 0.2550 & 0.206 & 0.2496 \\

ATE
& 3.857 & 23.505 & 0.0150 & 0.305 & 0.0074 \\
\bottomrule
\end{tabular}
\end{table}

\FloatBarrier

\section*{Author Contributions}
Ali Haghpanah Jahromi conceived the research idea; designed and implemented the full methodology and evaluation pipeline; collected and curated the data; established the benchmark protocols; implemented and ran the comparison methods; collected and organized the results; performed the experiments and the preliminary and final analyses; incorporated the received feedback and revised the manuscript accordingly; and wrote the original manuscript. Mohammad Taheri and Zohreh Azimifar supervised and directed the research, advised on the methodological and experimental design, reviewed and interpreted the results, and provided scientific editing and constructive feedback on the manuscript.

\vspace{-12pt}
\section*{Data and Code Availability}

Benchmark data sources are cited in \cref{sec:datasets}. The accompanying
materials provide task-level metric exports, full-precision comparison
matrices, configuration and partition summaries, optimizer diagnostics,
and a script reproducing the statistical analyses from those exports.
Original raw benchmark payloads must be obtained from their respective
releases. Metric-level reconstruction does not replace end-to-end training
reproduction. The supplementary audit identifies the additional provenance
required for historical aggregation controls.

\vspace{-12pt}
\section*{Acknowledgment}

ChatGPT (OpenAI) assisted with English-language editing,
software-engineering support, preparation of analysis scripts
and tables, and drafting descriptions of the statistical results.
The authors take responsibility for reviewing the computations,
interpretations, and final manuscript. The supplied archive permits the
  verification described in the supplementary reproducibility audit.

\end{document}